\documentclass[11pt]{article}

\usepackage[final]{acl}

\usepackage{amsmath,amsfonts,bm}

\def\eqref#1{equation~\ref{#1}}
\def\1{\bm{1}}

\DeclareMathAlphabet{\mathsfit}{\encodingdefault}{\sfdefault}{m}{sl}
\SetMathAlphabet{\mathsfit}{bold}{\encodingdefault}{\sfdefault}{bx}{n}

\usepackage{times}
\usepackage{latexsym}
\usepackage{cuted}
\usepackage{caption}
\usepackage{wrapfig}
\usepackage{algorithm}
\usepackage{algorithmicx}
\usepackage{algpseudocode}
\usepackage{colortbl}

\usepackage[most]{tcolorbox}
\usepackage{mdframed}
\usepackage{fbox}
\usepackage{multirow}
\usepackage{booktabs}
\usepackage{enumitem}
\usepackage{graphicx}
\usepackage{pgfplots}
\pgfplotsset{compat=1.18}
\usepackage{tikz}
\usepackage{subcaption}
\usepackage[RGB]{xcolor}
\usepackage{amsthm}
\newtheorem{theorem}{Theorem} 
\newtheorem{assumption}{Assumption}

\usepackage{array}
\usepackage{listings}

\definecolor{col}{HTML}{598BE7}
\definecolor{col2}{HTML}{F54254}
\definecolor{col3}{RGB}{66, 244, 133}
\definecolor{RED}{RGB}{241, 79, 33}
\definecolor{GREEN}{RGB}{126, 185, 0}
\definecolor{BLUE}{RGB}{0, 163, 238}
\definecolor{YELLOW}{RGB}{254, 184, 0}
\definecolor{GRAY}{RGB}{114, 114, 114}
\definecolor{POSTECH}{RGB}{200, 1, 80}

\usetikzlibrary{patterns}
\newcolumntype{R}[1]{>{\raggedleft\arraybackslash}p{#1}}
\newcolumntype{L}[1]{>{\raggedright\arraybackslash}p{#1}}

\newcommand{\ie}{{\it i.e.,}\;}
\newcommand{\eg}{{\it e.g.,}\;}

\hypersetup{
  colorlinks=true,
  linkcolor=col,   
  citecolor=col,    
  urlcolor=col    
}

\lstdefinestyle{pythonstyle}{
    language=Python,           % Set the default language to Python
    basicstyle=\ttfamily\small,% Font style for the code
    keywordstyle=\color{black}, % Style for keywords (e.g., 'str', 'None', 'True')
    stringstyle=\color{orange!60!black},   % Style for string literals (e.g., '"step 1"')
    commentstyle=\color{green!50!black}, % Style for comments (though not directly used in your provided plan string)
    numbers=left,              % Add line numbers to the left
    numberstyle=\tiny\color{gray}, % Style for line numbers
    breaklines=true,           % Allow lines to break automatically
    frame=single,              % Add a single frame around the code block
    rulecolor=\color{black},   % Color of the frame
    showstringspaces=false,    % Do not show a special symbol for spaces in strings
    tabsize=4,                 % Define tab size
    captionpos=b,              % Position caption at the bottom
    breakatwhitespace=false,   % Do not break lines only at whitespace
}

\lstdefinestyle{pythonstyle2}{
  language=Python,
  basicstyle=\ttfamily\tiny,
  keywordstyle=\color{blue}\bfseries,
  stringstyle=\color{orange!60!black},
  commentstyle=\color{green!50!black},
  numbers=none,
  breaklines=true,
  frame=none,
  showstringspaces=false,
  columns=fullflexible,
  xleftmargin = -12pt,
  xrightmargin = -12pt,
  framesep=0pt,
  framexleftmargin=-10pt,
  framexrightmargin=-10pt,
  aboveskip = -5pt,
  belowskip = -5pt,
  lineskip = -3pt,
  escapeinside={(@}{@)},
}

\lstdefinestyle{pythonstyle3}{
  backgroundcolor=\color{col!5},
  basicstyle=\ttfamily\small,
  keywordstyle=\color{black},
  stringstyle=\color{orange!60!black},
  commentstyle=\color{green!50!black},
  numbers=none,
  breaklines=true,
  frame=none,
  showstringspaces=false,
  columns=fullflexible,
  keepspaces=true,                 
  showspaces=false,
  framesep=0pt,
  belowskip = -6pt,
  escapeinside={(@}{@)},
}

\lstdefinestyle{pythonstyle4}{
  backgroundcolor=\color{col!5},
  language=Python,
  basicstyle=\ttfamily\small,
  keywordstyle=\color{blue}\bfseries,
  stringstyle=\color{orange!60!black},
  commentstyle=\color{green!50!black},
  identifierstyle=\color{black},
  numbers=none,
  breaklines=true,
  frame=none,
  showstringspaces=false,
  columns=fullflexible,
  keepspaces=true,                 
  showspaces=false,
  framesep=0pt,
  belowskip = -6pt,
  escapeinside={(@}{@)},
}

\usepackage[T1]{fontenc}
\usepackage[utf8]{inputenc}

\usepackage{microtype}

\usepackage{inconsolata}

\usepackage{graphicx}

\title{PaT: Planning-after-Trial for Efficient Test-Time Code Generation}%\vspace{0.5cm}}
\author{
Youngsik Yoon$^1$, Sungjae Lee$^1$, Seockbean Song$^2$, Siwei Wang$^3$, Wei Chen$^3$, Jungseul Ok$^{1,2}$\thanks{Corresponding author.}\\
$^1$Department of Computer Science and Engineering, POSTECH, South Korea\\
$^2$Graduate School of Artificial Intelligence, POSTECH, South Korea\\
$^3$Microsoft Research Asia, Beijing, China\\
\{ysyoon97, sungjaelee25, shinebobo, jungseul.ok\}@postech.ac.kr, \\
\{siweiwang, weic\}@microsoft.com
}

\begin{document}
\maketitle
\begin{abstract}
Beyond training-time optimization, scaling test-time computation has emerged as a key paradigm to extend the reasoning capabilities of Large Language Models (LLMs).
However, most existing methods adopt a rigid Planning-before-Trial (PbT) policy, which inefficiently allocates test-time compute by incurring planning overhead even on directly solvable problems.
We propose Planning-after-Trial (PaT), an adaptive policy for code generation that invokes a planner only upon verification failure.
This adaptive policy naturally enables a heterogeneous model configuration: a cost-efficient model handles generation attempts, while a powerful model is reserved for targeted planning interventions.
Empirically, across multiple benchmarks and model families, our approach significantly advances the cost-performance Pareto frontier.
Notably, our heterogeneous configuration achieves performance comparable to a large homogeneous model while reducing inference cost by approximately 69\%.
\end{abstract}

\begin{figure}[!t]
    \centering
    \includegraphics[width=0.95\linewidth, trim={0 0 0 0},clip]{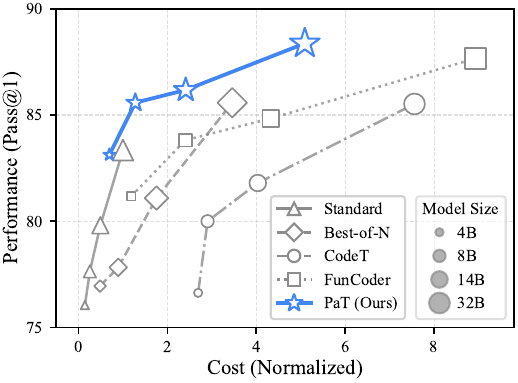}
    \vspace{-0.3cm}
    \caption{
    \textbf{Cost ($\downarrow$) - Pass@1 ($\uparrow$) trade-off across diverse sizes.}
    We plot the average Pass@1 across foundational benchmarks (HumanEval~\cite{chen2021evaluating}, MBPP~\cite{austin2021program} and their EvalPlus~\cite{evalplus} variants) against the relative inference cost.
    PaT consistently advances the Pareto frontier across model sizes (Qwen3$_\text{4B,8B,14B and 32B}$).
    Detailed results are provided in Section~\ref{sec:foundational} and Table~\ref{table:HumanEval}.
    }\label{fig:pareto}
    \vspace{-0.6cm}
\end{figure}

\begin{figure*}[!t]
    \centering
    \includegraphics[width=0.97\textwidth, trim={0 0 0 0},clip]{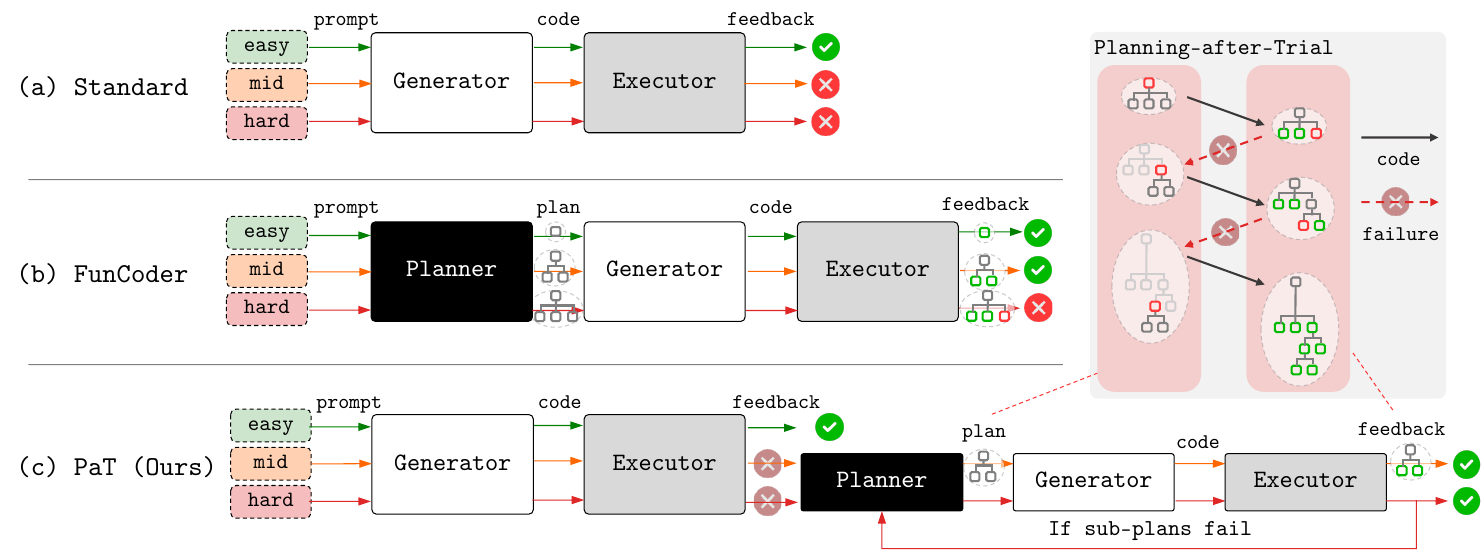}
    \vspace{-0.3cm}
    \caption{\textbf{Comparison with existing methods and PaT (Ours).}
    Problems are grouped by difficulty (easy, mid, and hard). 
    Boxes denote the key components: a Generator (creates code), a Planner (decomposes the problem), and an Executor (verifies the solution).
    \textbf{(a) Standard} directly generate and execute; works on easy problems but often fails on harder ones.
    \textbf{(b) FunCoder (PbT)} always plans first, so planning cost is paid even when unnecessary.
    \textbf{(c) PaT (Ours)} trials first and plans only on failure; solves easy problems cheaply and hard problems adaptively.
    }\label{fig:concept}
    \vspace{-0.5cm}
\end{figure*}

\section{Introduction}
\label{sec:introduction}

%% 1. code generation 엄청 중요하고, 연구많이 됨. code generation을 위해, SFT, RLHF 같은걸로 post-training 하는 것도 있지만, 동시에 test time scaling (or inference time scaling)도 연구 많이됨 (사람이 코딩하는 것처럼, writing, debugging,...) planning, reasoning 등의 복잡한 체계를 수반하는 scaling methods는 계산량 문제가 큼. 우리는 scaling methods 중에 cost-performance pareto front를 당기는걸 목표. (figure 1에서 좌상단으로...)

%% 2. computational efficiency를 위한 우리의 핵심 아이디어는, PaT임. 기존에는 planning이 선행되어 하나씩 코딩하는 식임. (예: funcoder). 엄청 쉬운것도 planning하고 generation하는 식으로해서 계산량 낭비됨. 예를들면... 73%....

%% 3. 

Large Language Models (LLMs) have achieved remarkable success in code generation, driven by massive model and data scaling \citep{chen2021evaluating, li2022competition, openai2023gpt4}.
Beyond scaling, conventional approaches focused on training-time strategies, such as Supervised Fine-Tuning (SFT) \citep{roziere2023code, luo2023wizardcoder} and Reinforcement Learning from Human Feedback (RLHF) \citep{ouyang2022training}.
Complementing these efforts, recent research scales test-time computation to handle complex algorithmic logic \citep{wei2022chain, yao2023tree, snell2025scaling}.
In particular, advancements have been made to simulate human-like coding behavior by decomposing complex problems into manageable sub-problems \citep{chen2024divide, le2024codechain} rather than generating the solution directly.

However, existing test-time scaling code generation methods often incur prohibitive inference cost.
This is primarily due to their heavy reliance on explicit planning modules \citep{chen2024divide, le2024codechain} or iterative repair and debugging loops \citep{shinn2023reflexion, zhong2024debug}, which consume significant computational resources.
As illustrated in Figure~\ref{fig:pareto}, previous state-of-the-art methods such as FunCoder~\citep{chen2024divide} consume excessive test-time resources,
often making a small model (Qwen3$_\text{4B}$) with this method more expensive than a much larger model (Qwen3$_\text{32B}$) with standard inference, even with inferior performance.
This indicates that previous methods underperform standard scaling in their cost-performance ratio.

This inefficiency stems from ignoring the intrinsic difficulty distribution of the workload. 
As shown in Figure~\ref{fig:pareto}, a small model (Qwen3$_\text{4B}$) already achieves 76\% Pass@1 using standard inference, suggesting that a substantial fraction of tasks can be solved without complex interventions.
Only a minority of hard instances truly require expensive planning.
Nonetheless, many prior approaches \citep{chen2024divide, le2024codechain, jiang2024self} follow a rigid ``Planning-before-Trial'' (PbT) policy, wasting resources by indiscriminately applying expensive planning to these tractable problems.

To address this, we introduce ``Planning-after-Trial'' (PaT), an adaptive method that inverts PbT.
Unlike PbT, which enforces planning for every problem, PaT initiates with standard inference and leverages execution feedback to verify this solution.
As illustrated in Figure~\ref{fig:concept}, expensive decomposition is triggered only when the initial attempt fails.
By strictly limiting planning to problems that demonstrably require it, PaT concentrates computational resources precisely on hard instances while keeping simple ones cheap.
Consequently, PaT significantly lowers the average inference cost while improving performance, thereby advancing the cost-performance Pareto frontier as shown in Figure~\ref{fig:pareto}.

The efficiency of PaT is further amplified within a heterogeneous model configuration.
This enables a strategic division of labor, where a cost-efficient model is assigned to the high-volume ``generator'' role and a powerful one to the infrequent ``planner'' role.
This alignment reflects the intuition that planning requiring more complex reasoning (\eg problem decomposition) is better handled by stronger models, while the resulting decomposed sub-tasks can often be implemented proficiently by smaller models.

We empirically validate our claims across diverse model families and a comprehensive suite of benchmarks.
Our results first establish that PaT consistently outperforms the state-of-the-art PbT baseline, FunCoder \citep{chen2024divide}, achieving higher performance across all model scales while using, on average, only 60\% of the inference cost.
Furthermore, we show that the heterogeneous model configuration further enhances this efficiency.
For instance, a small generator model guided by a powerful planner achieves competitive performance, with a <1\% gap to a larger model in the homogeneous setting, while incurring only 31\% of the cost.
These findings confirm that PaT's adaptive policy establishes a new, superior cost-performance frontier.
Collectively, our work demonstrates that scaling test-time compute need not be uniform, proving that unnecessary overhead can be effectively eliminated while maintaining or even enhancing performance.

Our contributions are summarized as follows:
\begin{itemize}[left=0.3cm]
\item We propose Planning-after-Trial (PaT), an adaptive policy that invokes planning only upon verification failure to avoid unnecessary overhead.
\item We pair PaT with a heterogeneous model configuration, combining small models for generation and large models for planning to maximize efficiency.
\item We provide a comprehensive evaluation across model families and benchmarks, showing that our approach consistently advances the cost-performance Pareto frontier.
\end{itemize}

\section{Related work}
\label{sec:relatedwork}

\paragraph{Code Generation with LLMs.}
To advance code generation, research has prioritized enhancing intrinsic capabilities via post-training \citep{dubey2024llama, zhu2024deepseek} and self-improvement frameworks like rStar-Coder \citep{liu2025rstar}, SPIN \citep{chen2024self}, and SCoRE \citep{kumar2024training}.
In parallel, significant attention has shifted toward scaling test-time reasoning and computation given models.
Key approaches utilize few-shot prompting \citep{gpt3} and diversity sampling \citep{chen2021evaluating}, often refined by test-driven verification \citep{chen2022codet, shinn2023reflexion}, iterative refinement \citep{wang2024intervenor, zhang2024pair, zhong2024debug}, or retrieval from external bases \citep{li2022competition, zhang2023repocoder}.
To tackle higher complexity, recent systems have further incorporated hierarchical decomposition \citep{chen2024divide, le2024codechain} and natural language planning \citep{wang2025planning}.
While these approaches improve robustness, they often apply sophisticated interventions uniformly across problems, incurring unnecessary overhead for simple cases.
Notably, AdaCoder \citep{zhu2025adacoder} dynamically selects refinement strategies but still relies on iterative repair rather than structural decomposition.
In contrast, PaT optimizes test-time computation, reserving complex interventions strictly for verified failures using execution feedback as a precise trigger.

\paragraph{Decomposition and Structured Reasoning.}
To tackle complex problems beyond the reach of direct generation, divide-and-conquer strategies have evolved from implicit reasoning steps in Chain-of-Thought \citep{wei2022chain} and Tree-of-Thoughts \citep{yao2023tree} to explicit planning paradigms.
This evolution spans diverse domains, ranging from sequential solving in Least-to-Most prompting \citep{zhou2023least} and mathematical reasoning with Program of Thoughts \citep{chen2022program}, to using a society of models for complex tasks \citep{juneja2024texttt}.
In the code domain, structured pipelines like Self-Plan \citep{jiang2024self} and FunCoder \citep{chen2024divide} explicitly decompose problems before implementation.
However, these approaches predominantly adhere to a rigid Planning-before-Trial (PbT) policy, where planning is invoked unconditionally prior to execution.
While recent work attempts to mitigate this inefficiency by learning an adaptive policy for invoking planning \citep{paglieri2025learning}, such approaches introduce additional training complexity and auxiliary models.
In contrast, PaT offers a simpler, learning-free alternative: it utilizes the definitive execution signal as a reactive trigger, avoiding both the universal cost of PbT and the overhead of training separate policy models.

\section{PaT: Planning after Trial}
\label{sec:method}
\paragraph{Problem Formulation.}
We model a code generation instance as a specification $x$ (e.g., natural language description) and seek a program $\mathcal{F}$ that satisfies $x$.
Let $M_G$ denote a generator model and $M_P$ a planner model.
For any specification $x$, the generator $M_G$ produces a direct candidate implementation, denoted $\hat{f}$.
The planner $M_P$, when invoked, produces a decomposition plan consisting of a new top-level implementation, $\hat{f}$, and a set of subproblem specifications $\{x_i\}$. The final program $\mathcal{F}$ is constructed by $\textsc{Compose}$ that merges an implementation $\hat{f}$ with a set of verified helper functions $H$.
For verification, we construct a test set $\mathcal{T}(x)=\{(\mathrm{in}_j,\mathrm{out}_j)\}_{j=1}^t$ and execute programs in a sandboxed Python runtime.
We evaluate $\mathcal{F}$ with
\begin{align}    
\!\!\!\textsc{Evaluate}\big(\mathcal{F},\mathcal{T}(x)\big)
\!=\!\!
\sum_{j=1}^{t}\mathbf{1}\!\big[\mathcal{F}(\mathrm{in}_j)\!=\!\mathrm{out}_j\big],\!
\end{align}
\ie the number of tests passed.

\begin{algorithm}[t]
    \caption{\textsc{Planning after Trial (PaT)}}
    \label{alg:pat}
    \begin{algorithmic}[1]
        \State \textbf{Input:} Problem $x$, Set of Helper Functions $H$,
        \State \quad\quad\quad Generator $M_G$, Planner $M_P$
        \State \textbf{Output:} Generated Program $\mathcal{F}$
        \State \textbf{PaT:}
        \State $\hat{f} \gets M_G(x;H)$ \Comment{Best-of-$N$ trial}
        \State $\mathcal{T}(x) \gets \textsc{GenerateTests}(x)$
        \State $p \gets \textsc{Evaluate}(\textsc{Compose}(\hat{f},H), \mathcal{T}(x))$
        
        \If{$p = |\mathcal{T}(x)|$}
            \State $\mathcal{F} \gets \textsc{Compose}(\hat{f},H)$
            \State \textbf{return} $\mathcal{F}$ \Comment{Trial success}
        \EndIf
        
        \While{True}\Comment{Until success or plateau}
            \State $\hat{f}_\mathrm{prev}, p_\mathrm{prev} \gets \hat{f}, p$
            \State $\hat{f}, \{x_i\} \gets M_P(x; H)$ \Comment{Planner invoked}
            \For{each $x_i$ not implemented in $H$}
                \State $\hat{f}_i \gets \textsc{PaT}(x_i, H, M_G, M_P)$
                \State $H \gets H \cup \{\hat{f}_i\}$
            \EndFor
            
            \State $p \gets \textsc{Evaluate}(\textsc{Compose}(\hat{f},H), \mathcal{T}(x))$
            
            \If{$p = |\mathcal{T}(x)|$}\Comment{Final success}
                \State \textbf{return} $\textsc{Compose}(\hat{f},H)$ 
            \EndIf
            
            \If{$p \le p_\mathrm{prev}$}\Comment{Plateau condition}
                \State \textbf{return} $\textsc{Compose}(\hat{f}_\mathrm{prev},H)$ 
            \EndIf
        \EndWhile
    \end{algorithmic}
\end{algorithm}
% \caption{\textbf{Pseudocode for Planning after Trial (PaT).}
% \textsc{GenerateTests}$(x)$ constructs the unit-test set $\mathcal{T}(x)$ from the specification. \textsc{Compose}$(\hat{f},H)$ merges verified helper implementations into the parent program. The planner is invoked only on failure, and verified subsolutions are reused as context.}
% \label{fig:pseudo_code_clean}

\subsection{Adaptive Planning via Failure Feedback}
\label{subsec:adaptive-decomposition}

PaT executes a simple yet effective failure-triggered policy.
Given the problem specification $x$, the generator $M_G$ first attempts a Best-of-$N$ trial by sampling multiple candidate solutions.
We verify these candidates on a test set $\mathcal{T}(x)$. 
If any candidate passes all tests, the pipeline terminates and returns that solution immediately without incurring any decomposition overhead.
However, if all candidates fail, this consistent failure serves as a strong signal that the problem complexity exceeds the generator's direct reasoning capability, rather than being a simple implementation error.
Only in this case, PaT invokes the planner $M_P$ to decompose $x$ into subproblems $\{x_i\}$, which are solved recursively using the same trial-first policy.

When all subproblems have verified solutions, \ie $\textsc{Evaluate}(\hat{f}_i,\mathcal{T}(x_i)) = |\mathcal{T}(x_i)|$ for all $i$, PaT composes them into a parent-level candidate and re-verifies against $\mathcal{T}(x)$.
If this composite solution passes, the process terminates successfully.
Otherwise, the planner is invoked again on the original problem specification $x$.
For this subsequent planning attempt, the planner is provided with the original problem and the set of previously successful subsolutions $\{\hat{f}_i\}$ as additional context.
This allows the planner to make a more informed decomposition decision while reusing already successful components, further enhancing cost-efficiency.

\subsection{Test Cases and Verification}
\label{subsec:testcases}
Relying solely on provided public test cases is often insufficient.
Benchmarks typically provide only basic scenarios that lack critical edge cases \citep{chen2021evaluating}, or in some cases like MBPP, provide none at all \citep{austin2021program}.
To ensure robust operation in such environments, we explicitly generate test cases $\mathcal{T}(x)$, producing an average of 6.7 test cases per problem.

\begin{figure}[!t]
    \centering
    \begin{tikzpicture}
        \begin{axis}[
            ybar stacked,
            bar width=1cm,
            width=0.45\textwidth,
            height=5cm,
            ymin=0, ymax=100,
            ylabel={Percentage (\%)},
            ylabel style={font=\footnotesize},
            yticklabel style={font=\footnotesize},
            symbolic x coords={Qwen3$_{4B}$, Qwen3$_{8B}$, Qwen3$_{14B}$, Qwen3$_{32B}$},
            xtick=data,
            xticklabel style={align=center, font=\footnotesize},
            legend style={
                at={(0.5,1.05)}, 
                anchor=south, 
                legend columns=-1, 
                draw=none, 
                font=\footnotesize,
                /tikz/every even column/.append style={column sep=0.3cm}
            },
            nodes near coords={%
                \pgfmathprintnumber[fixed, precision=1]{\pgfplotspointmeta}\%%
            },
            every node near coord/.append style={
                font=\bfseries\sffamily\scriptsize, 
                color=black, 
                % yshift=-2pt % 숫자 위치 미세 조정
            },
            % 색상 설정: 0개(초록) -> 1개(노랑) -> 2개(주황) -> 3개+(빨강)
            cycle list={
                {fill=col!70, draw=black!50},
                {fill=col!50, draw=black!50},
                {fill=col!30, draw=black!50},
                {fill=col!10, draw=black!50}
            }
        ]
            % 1. Zero Incorrect (완벽함)
            \addplot coordinates {
                (Qwen3$_{4B}$, 63.4) 
                (Qwen3$_{8B}$, 55.7) 
                (Qwen3$_{14B}$, 70.3) 
                (Qwen3$_{32B}$, 66.5)
            };
            
            % 2. One Incorrect (사소한 오류)
            \addplot coordinates {
                (Qwen3$_{4B}$, 11.6) 
                (Qwen3$_{8B}$, 20.9) 
                (Qwen3$_{14B}$, 14.6) 
                (Qwen3$_{32B}$, 15.2)
            };
            
            % 3. Two Incorrect
            \addplot coordinates {
                (Qwen3$_{4B}$, 14.6) 
                (Qwen3$_{8B}$, 8.9) 
                (Qwen3$_{14B}$, 8.9) 
                (Qwen3$_{32B}$, 12.2)
            };
            
            % 4. Three or more Incorrect (심각한 노이즈)
            \addplot coordinates {
                (Qwen3$_{4B}$, 10.4) 
                (Qwen3$_{8B}$, 14.5) 
                (Qwen3$_{14B}$, 6.2) 
                (Qwen3$_{32B}$, 6.1)
            };
            
            \legend{0, 1, 2, 3+}
        \end{axis}
    \end{tikzpicture}
    \vspace{-0.2cm}
    \caption{\textbf{Distribution of incorrect test cases (false positives) per problem on HumanEval.} The percentage of problems where the generated test cases contain 0, 1, 2, or 3+ incorrect test cases, given an average of 6.7 generated test cases per problem.}
    \label{fig:test_noise_distribution}
    \vspace{-0.3cm}
\end{figure}
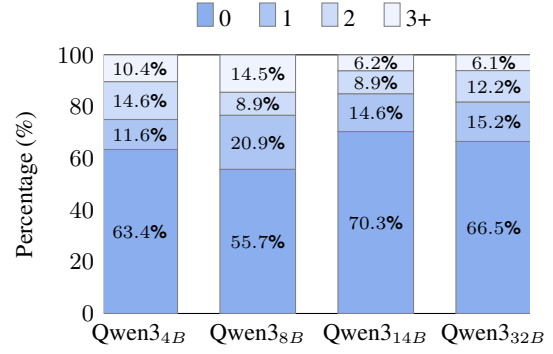

However, prior works note that generated test cases can be noisy or incorrect \citep{chen2022codet, wang2025testeval, prasad2025learning}.
To mitigate this, most previous works adopt consensus-based scoring to identify the most robust output \citep{chen2022codet, chen2024divide}. 
This approach evaluates a pool of candidate outputs against generated test inputs, selecting the one that achieves the broadest consensus.
However, such consensus-based scoring is ill-suited for our framework, as its objective of selecting the most likely correct outputs contrasts with PaT's need for a definitive binary signal (success or failure) to trigger its adaptive escalation to planning.

We therefore separate the success criterion from the termination rule.
The success criterion is strict: a candidate is accepted only if it passes all test cases in $\mathcal{T}(x)$.
This reliance on a strict signal is feasible because, as shown in Figure~\ref{fig:test_noise_distribution}, generated test cases are completely error-free for the majority of problems (e.g., 63.4\% on HumanEval with Qwen3$_\text{4B}$).
Furthermore, even in noisy instances, errors are typically limited to 1-2 false positives, with severe noise (3+ errors) concentrated in a small minority.
To handle these residual noisy test cases and prevent unproductive recursion, we track the pass count $p^{(t)}=\textsc{Evaluate}(\hat f,\mathcal{T}(x))$ at iteration $t$ and apply a plateau rule:
if $p^{(t)} \le p^{(t-1)}$ then halt and return the best-known solution.
We find this heuristic to be practically effective, as it prevents the model from overfitting to potential false positives by terminating the search when performance saturates.
% Intuitively, if the pass count does not strictly improve, we assume all valid test cases have been satisfied and remaining failures are due to noisy test cases.
% To avoid overfitting to potentially incorrect failing test cases, we do not provide them as explicit feedback to the generator.

% We therefore separate the success criterion from the termination rule.
% The success criterion is strict: a candidate is accepted only if it passes all tests in \(\mathcal{T}(x)\).
% To handle noisy tests and prevent unproductive recursion, we track the pass count \(p^{(t)}=\textsc{Evaluate}(\hat f,\mathcal{T}(x))\) at iteration $t$ and apply a plateau rule:
% if $p^{(t)} \le p^{(t-1)}$ then halt and return the best-known solution.
% Intuitively, if the pass count does not strictly improve, we assume all valid tests have been satisfied and remaining failures are due to flawed tests.
% To avoid overfitting to potentially incorrect failing tests, we do not provide those failing cases as explicit feedback to the generator.

\subsection{Heterogeneous Model Configuration}
\label{sec:heterogeneous}
To further enhance efficiency, we adopt a heterogeneous model configuration that assigns distinct models to the two key roles inherent in PaT: the generator and the planner.
These roles place different demands on model capability.
The generator's task, which is to produce a solution for a well-scoped and often manageable subproblem, can be effectively handled by a cost-efficient small language model (sLM).
In contrast, the planner's task, which is to understand the nuances of a complex specification and propose a helpful decomposition, benefits from the advanced reasoning of a high-performance large language model (LLM).
Building on this observation, we instantiate PaT with an sLM as the generator and an LLM as the planner.

However, achieving efficiency requires more than simply ``using a smaller model as the generator.''
While decomposition is less expensive than full generation, it still incurs a planning token cost.
We must therefore navigate a critical trade-off: an overly weak sLM triggers frequent failures, necessitating repeated interventions by the planner, which raises the total cost.
Conversely, an overly strong sLM reduces the need for planning but increases the baseline cost of every trial, eroding the benefits of heterogeneity.
Therefore, maximizing efficiency requires empirically tuning the generator selection to strike a balance between minimizing baseline costs and avoiding excessive planning overhead.
To provide a rigorous foundation for these observations, we provide a formal analysis of this trade-off in Appendix~\ref{sec:app_theory}.
% Therefore, the optimal strategy lies in selecting an sLM that is sufficiently capable of solving simplified subproblems, thereby reserving the planner's expensive compute solely for essential structural decompositions.
\section{Experiments}
\label{sec:experiments}
In this section, we evaluate PaT in terms of both performance and cost efficiency.
We conduct experiments on two settings: a homogeneous setting (Section~\ref{sec:homogeneous}), and a heterogeneous setting (Section~\ref{sec:hybrid-exp}).
Below we describe the experimental setup used in our evaluation, including benchmarks, LLMs, evaluation metrics, and baselines.

\paragraph{Benchmarks.}
We evaluate on established code generation benchmarks. 
\textbf{HumanEval} \citep{chen2021evaluating} and \textbf{MBPP} \citep{austin2021program} measure foundational code generation.
For both, we use EvalPlus \citep{evalplus} to obtain expanded and more robust unit tests (\textit{i.e.}, \textbf{HumanEval+} and \textbf{MBPP+}).
For challenging tasks, we use \textbf{xCodeEval} \citep{khan2023xcodeeval} benchmark, partitioning instances into four categories (Easy, Mid, Hard, and Expert) following FunCoder's rating scheme.

\paragraph{LLMs.}
We primarily adopt the Qwen3 family (4B, 8B, 14B, and 32B) \citep{yang2025qwen3} as our main testbed, since it provides fine-grained scaling steps within a single model family. 
To evaluate the cross-family generalization of our method, 
we additionally conduct experiments on Llama-3.1-8B-Instruct \citep{dubey2024llama} 
and DeepSeek-Coder-V2-Lite-Instruct ($\approx$16B) \citep{zhu2024deepseek}.
For the cost analysis, we calculate inference costs based on public market pricing as detailed in Appendix~\ref{sec:app_LLMs}.

\paragraph{Evaluation Metric.}
We report two primary metrics: \textbf{Pass@1} and \textbf{LLM cost}. 
Pass@1 is our primary performance metric, reflecting the ability to generate a correct solution on the first attempt, which aligns with real-world efficiency goals.
LLM cost for each experiment is computed directly from the per-token prices in Table~\ref{table:LLM_cost}.

\paragraph{Baselines.}
To validate the effectiveness of PaT, we compare it against the following state-of-the-art algorithms for scaling test-time computation:

\begin{itemize}[left=0.3cm]
\item \textbf{Standard} \citep{gpt3} uses few-shot prompting, giving several in-context examples to the model and asking it to generate a solution; we follow the original prompting protocol to measure baseline quality without post-hoc filtering or decomposition.
\item \textbf{Best-of-N} \citep{chen2021evaluating} generates multiple candidates (N=5) and selects the best one; we include it to distinguish the gains of our adaptive planning from those achievable simply by scaling inference compute.
\item \textbf{CodeT} \citep{chen2022codet} performs consensus-based selection by executing solutions against automatically generated test cases; it serves as a representative benchmark for the efficacy of test-driven verification methods.
\item \textbf{FunCoder} \citep{chen2024divide} implements a static divide-and-conquer pipeline with a fixed hierarchy, solves subproblems, and assembles the final solution; we compare against it to highlight the efficiency gains of our adaptive, failure-driven policy over rigid Planning-before-Trial strategies.
\end{itemize}

For the sake of reproducibility and clarity, a detailed account of our implementation is included in Appendix~\ref{sec:app_implementation}--\ref{sec:app_prompt}.

\subsection{Homogeneous Setting}
\label{sec:homogeneous}
In this section, we compare results from running PaT using a homogeneous language model against baselines. 
Evaluations are performed on the foundational benchmarks (HumanEval and MBPP; Sec.~\ref{sec:foundational}) and the challenging benchmark (xCodeEval; Sec.~\ref{sec:complex}).

\begin{table*}[!t]
\centering

\footnotesize
\begin{tabular}{L{2cm}L{1.5cm}R{0.8cm}R{0.8cm}R{0.8cm}R{0.8cm}R{0.8cm}R{1.0cm}R{0.8cm}}%R{0.7cm}}
\toprule
\multirow{2}{*}{\textbf{Model}} & \multirow{2}{*}{\textbf{Method}} & \multicolumn{2}{c}{\textbf{HumanEval}} & \multicolumn{2}{c}{\textbf{MBPP}} & \multirow{2}{*}{\textbf{Avg.}} & \multirow{2}{*}{$\Delta$ \textbf{Avg.}}& \multirow{2}{*}{\textbf{Cost}}% & \multirow{2}{*}{\textbf{Eff.}}
\\
\cmidrule(lr){3-4} \cmidrule(lr){5-6}
& & \multicolumn{1}{c}{base} & \multicolumn{1}{c}{plus} & \multicolumn{1}{c}{base} & \multicolumn{1}{c}{plus} & & & 
\\
\midrule
\multirow{5}{*}{Qwen3$_\text{4B}$} 
& Standard & 78.66 & 70.12 & 79.43 & 76.00 & 76.05 & - & 1.00 %& - 
\\
& Best-of-N & 79.27 & 71.95 & 80.00 & 76.57 & 76.95 & + 0.90 & 3.39 %& 0.38
\\
& CodeT & 78.66 & 70.73 & 80.00 & 77.14 & 76.63 & + 0.58 & 18.82  %& - 0.07
\\
& FunCoder & {85.98} & {79.88} & {81.14} & {77.71} & {81.18} & + 5.13 & 8.31 %& {0.70}
\\
& PaT (Ours) & \textbf{89.63} & \textbf{82.32} & \textbf{82.29} & \textbf{78.29} & \textbf{83.13} & + 7.08 & 4.85 %& \textbf{1.84}
\\
\cmidrule(lr){1-9}
\multirow{5}{*}{Qwen3$_\text{8B}$} 
& Standard & 78.66 & 71.34 & 81.14 & 79.43 & 77.64 & - & 1.00 %& - 
\\
& Best-of-N & 80.49 & 72.56 & 80.57 & 77.71 & 77.83 & + 0.19 & 3.50 %& 0.08
\\
& CodeT & 80.49 & 73.78 & {84.57} & 81.14 & 80.00 & + 2.36 & 11.37 %& 0.23
\\
& FunCoder & {89.02} & {81.10} & 83.43 & {81.71} & {83.82} & + 6.18 & 9.43 %& {0.73}
\\
& PaT (Ours) & \textbf{90.85} & \textbf{82.32} & \textbf{85.14} & \textbf{84.00} & \textbf{85.58} & + 7.94 & 5.00 %& \textbf{1.98}
\\
\cmidrule(lr){1-9}
\multirow{5}{*}{Qwen3$_\text{14B}$} 
& Standard & 83.54 & 76.83 & 80.57 & 78.29 & 79.81 & - & 1.00 %& - 
\\
& Best-of-N & 82.93 & 77.44 & 82.86 & 81.14 & 81.09 & + 1.28 & 3.58 %& 0.50
\\
& CodeT & 83.54 & 76.22 & \textbf{85.71} & 81.71 & 81.80 & + 1.99 & 8.22 %& 0.28
\\
& FunCoder & {89.63} & {81.71} & {85.14} & {82.86} & {84.84} & + 5.03 & 8.82 %& {0.64}
\\
& PaT (Ours) & \textbf{91.46} & \textbf{83.54} & \textbf{85.71} & \textbf{84.00} & \textbf{86.18} & + 6.37 & 4.91 %& \textbf{1.48}
\\
\cmidrule(lr){1-9}
\multirow{5}{*}{Qwen3$_\text{32B}$} 
& Standard & 87.20 & 80.49 & 83.43 & 82.29 & 83.35 & - & 1.00 %& - 
\\
& Best-of-N & 90.24 & 82.93 & 85.14 & 84.00 & 85.58 & + 2.23 & 3.46 %& 0.44
\\
& CodeT & 89.02 & 80.49 & {86.86} & {85.71} & 85.52 & + 2.17 & 7.56 %& 0.16
\\
& FunCoder & {93.29} & \textbf{84.76} & {86.86} & {85.71} & {87.66} & + 4.31 & 8.93 %& {0.47}
\\
& PaT (Ours) & \textbf{93.90} & {84.15} & \textbf{88.57} & \textbf{86.86} & \textbf{88.37} & + 5.02 & 5.09 %& \textbf{0.81}
\\
\midrule
\multirow{5}{*}{Llama3.1$_\text{8B}$} 
& Standard & 68.29 & 59.15 & 66.86 & 66.29 & 65.15 & - & 1.00 %& - 
\\
& Best-of-N & 70.12 & 62.20 & {72.57} & 69.14 & 68.51 & + 3.36 & 2.60 %& - 1.37
\\
& CodeT & 71.95 & 61.59 & 72.00 & {71.43} & 69.24 & + 4.09 & 4.86 %& {1.06}
\\
& FunCoder & {75.00} & {67.68} & 72.00 & {71.43} & {71.53} & + 6.38 & 8.07 %& 0.90
\\
& PaT (Ours) & \textbf{78.05} & \textbf{68.90} & \textbf{74.29} & \textbf{72.00} & \textbf{73.31} & + 8.16 & 5.32 %& \textbf{1.89}
\\
\cmidrule(lr){1-9}
\multirow{5}{*}{DeepSeek-Coder} 
& Standard & 83.54 & 75.61 & 80.57 & 78.86 & 79.65 & - & 1.00 %& - 
\\
& Best-of-N & 81.10 & 75.00 & 81.71 & 81.14 & 79.74 & + 0.09 & 3.05 %& 0.04
\\
& CodeT & 81.10 & 72.56 & {84.57} & 82.86 & 80.27 & + 0.62 & 6.40 %& 0.11
\\
& FunCoder & {85.98} & {79.27} & \textbf{85.14} & {84.00} & {83.60} & + 3.95 & 8.77 %& {0.51}
\\
& PaT (Ours) & \textbf{86.59} & \textbf{79.88} & \textbf{85.14} & \textbf{85.14} & \textbf{84.19} & + 4.54 & 5.97 %& \textbf{0.91}
\\
\bottomrule
\end{tabular}
\vspace{-0.2cm}
\caption{\textbf{Performance and cost comparison on foundational benchmarks.}
We report Pass@1 across four benchmarks (HumanEval, HumanEval+, MBPP, and MBPP+) and their average (`Avg.'). Cost is normalized relative to the Standard baseline (1.00). Best results are in \textbf{bold}.
}
\label{table:HumanEval}
\end{table*}

\subsubsection{Foundational Benchmarks}
\label{sec:foundational}
We evaluate PaT and baselines on the foundational benchmarks (HumanEval and MBPP) and their EvalPlus variants.
As shown in Table~\ref{table:HumanEval}, PaT consistently outperforms all baselines across all tested model families and scales.
The effectiveness of this policy is best illustrated by its ability to dramatically improve the capability of smaller models.
For instance, PaT enables Qwen3$_\text{4B}$ to achieve an average Pass@1 of 83.13\%, which is remarkably similar to the 83.35\% achieved by Standard on Qwen3$_\text{32B}$, a model eight times larger.
This demonstrates that PaT is not merely an incremental improvement but a powerful policy that fundamentally alters the performance curve of a given model.

% Beyond performance, PaT demonstrates superior cost-efficiency, outperforming FunCoder, the previous state-of-the-art test-time code generation method, while consuming only 60\% of its cost.
% This efficiency stems from the fact that Standard solves $\sim$76\% of problems directly; thus, FunCoder's rigid policy incurs unnecessary overhead on the majority of tasks.
% In contrast, PaT restricts planning to the $\sim$24\% of instances that fail the initial trial.
% By avoiding universal planning overhead, PaT establishes a new and more effective cost-performance frontier.

Beyond its superior performance, PaT also shows critical advantages in cost-efficiency.
As detailed in Table~\ref{table:HumanEval}, PaT consistently achieves higher performance than FunCoder, the previous state-of-the-art hierarchical method, while consuming, on average 60\% of its cost.
This efficiency gain can be explained by examining Standard performance.
On these foundational benchmarks, Standard solves, on average, 76\% of problems directly, demonstrating that a large majority of tasks do not require decomposition.
FunCoder's rigid PbT policy is forced to pay a cost of planning on all of these problems, incurring its expensive planning overhead even when it is not needed.
In contrast, PaT only invokes its planner on the much smaller fraction of problems (24\%) that actually fail the initial, cheaper trial.
By avoiding this universal planning overhead, PaT establishes a new and more effective cost-performance frontier for code generation.

\begin{table*}[!t]
\centering
\footnotesize
\begin{tabular}{L{2cm}L{1.5cm}R{0.8cm}R{0.8cm}R{0.8cm}R{0.8cm}R{0.8cm}R{1.0cm}R{0.8cm}}
\toprule
\textbf{Model} & \textbf{Method} & \textbf{Easy} & \textbf{Mid} & \textbf{Hard} & \textbf{Expert} & \textbf{All} & $\Delta$ \textbf{All} & \textbf{Cost} \\
\midrule
\multirow{5}{*}{Qwen3$_\text{4B}$} 
& Standard & 37.70 & 17.86 & 3.45 & 0.00 & 18.40 & - & 1.00  \\
& Best-of-N & 51.91 & {29.46} & 8.05 & 0.00 & 27.00 & + 8.60 & 4.06  \\
& CodeT & 40.44 & 20.54 & 3.45 & 0.00 & 20.00 & + 1.60 & 21.64  \\
& FunCoder & {55.19} & {29.46} & {12.64} & 0.00 & {29.00} & + 10.60 & 12.95  \\
& PaT (Ours) & \textbf{61.75} & \textbf{40.18} & \textbf{14.94} & 0.00 & \textbf{34.20} & + 16.20 & 17.93 \\
\cmidrule(lr){1-9}
\multirow{5}{*}{Qwen3$_\text{8B}$} 
& Standard & 54.10 & 28.57 & 5.75 & 0.00 & 27.20 & -  & 1.00 \\
& Best-of-N & {66.67} & 42.86 & 6.90 & 0.00 & {35.20} & + 8.00 & 3.34 \\
& CodeT & 45.92 & 31.25 & 8.05 & 0.00 & 25.20 & - 2.00 & 17.00 \\
& FunCoder & 64.48 & {43.75} & {9.20} & 0.00 & 35.00 & + 7.80 & 8.62 \\
& PaT (Ours) & \textbf{69.95} & \textbf{45.54} & \textbf{11.49} & 0.00 & \textbf{37.80} & + 10.60 & 6.98 \\
\cmidrule(lr){1-9}
\multirow{5}{*}{Qwen3$_\text{14B}$} 
& Standard & 53.55 & 36.61 & 9.20 & 0.00 & 25.20 & - & 1.00  \\
& Best-of-N & 68.31 & 50.00 & 13.79 & 0.00 & 38.60 & + 13.40 & 3.16 \\
& CodeT & 55.19 & 41.07 & 9.20 & 0.00 & 31.00 & + 5.80 & 7.48 \\
& FunCoder & {73.22} & {52.68} & {18.39} & 0.00 & {41.80} & + 16.60 & 9.03  \\
& PaT (Ours) & \textbf{73.77} & \textbf{53.57} & \textbf{21.84} & \textbf{0.85} & \textbf{43.00} & + 17.80 & 6.49 \\
\cmidrule(lr){1-9}
\multirow{5}{*}{Qwen3$_\text{32B}$} 
& Standard & 54.64 & 39.29 & 11.49 & 0.00 & 30.80 & - & 1.00\\
& Best-of-N & 71.04 & {50.00} & 14.94 & 0.00 & 39.80 & + 9.00 & 3.28 \\
& CodeT & 57.92 & 39.29 & 12.64 & 0.00 & 32.20 & + 1.40 & 7.55 \\
& FunCoder & \textbf{74.86} & \textbf{54.46} & {16.09} & 0.00 & {42.40} & + 11.60 & 7.87 \\
& PaT (Ours) & {74.32} & \textbf{54.46} & \textbf{18.39} & \textbf{1.69} & \textbf{43.00} & + 12.20 & 6.00 \\
\midrule
\multirow{5}{*}{Llama3.1$_\text{8B}$} 
& Standard & 13.11 & 3.57 & {1.15} & 0.00 & 5.80 & - & 1.00\\
& Best-of-N & 22.40 & {8.04} & {1.15} & 0.00 & 10.20 & + 4.40 & 1.39 \\
& CodeT & 13.11 & 3.57 & {1.15} & 0.00 & 5.80 & + 0.00 & 2.59  \\
& FunCoder & {26.78} & 7.14 & \textbf{2.30} & 0.00 & {11.80} & + 6.00 & 5.46  \\
& PaT (Ours) & \textbf{32.79} & \textbf{11.61} & \textbf{2.30} & 0.00 & \textbf{15.00} & + 14.20 & 8.42 \\
\cmidrule(lr){1-9}
\multirow{5}{*}{DeepSeek-Coder} 
& Standard & 43.17 & 19.64 & 9.20 & {0.85} & 22.00 & - & 1.00\\
& Best-of-N & 59.56 & 26.79 & 10.34 & {0.85} & 29.80 & + 7.80 & 2.57 \\
& CodeT & 46.99 & 22.32 & 9.20 & 0.00 & 23.80 & + 1.80 & 5.46 \\
& FunCoder & {62.30} & {31.25} & {12.64} & \textbf{1.69} & {32.40} & + 10.40 & 8.69 \\
& PaT (Ours) & \textbf{63.39} & \textbf{33.93} & \textbf{13.79} & \textbf{1.69} & \textbf{33.60} & + 11.60 & 7.69 \\
\bottomrule
\end{tabular}
\vspace{-0.2cm}
\caption{\textbf{Performance breakdown on xCodeEval by difficulty.}
We report Pass@1 scores for each difficulty category (Easy, Mid, Hard, and Expert) and the overall score.
}
\label{table:xCodeEval}
\end{table*}
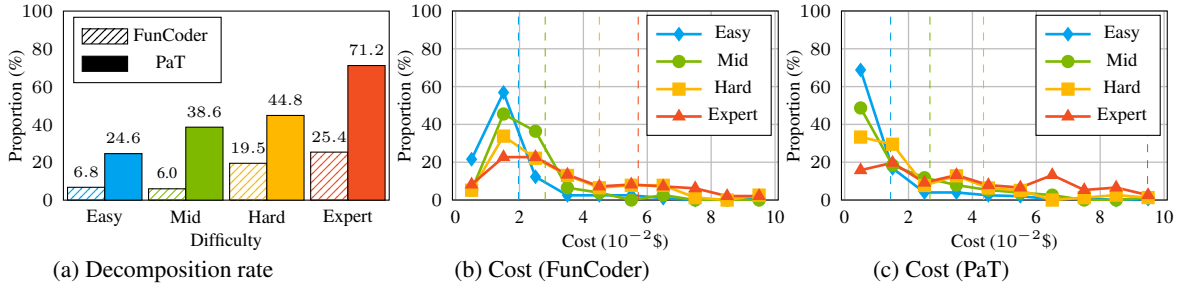
\begin{figure*}[!t]   % r=오른쪽, 너비=텍스트폭의 45%
    % \centering
    \begin{subfigure}[!t]{0.272\textwidth}
    % \hspace*{-0.8cm} 
    \begin{tikzpicture}
  \begin{axis}[
      ybar,
      bar width=14pt,
      ymin=0,
      ymax=100,
      ylabel={Proportion (\%)},
      xlabel={Difficulty},
      ylabel style={yshift=-9pt},
      xlabel style={yshift=5pt},
      tick label style={font=\scriptsize},
      label style={font=\scriptsize},
      legend style={font=\scriptsize},
      symbolic x coords={Easy, Mid, Hard, Expert},
      xtick={Easy, Mid, Hard, Expert},
      scale only axis,
      tick style={draw=none},
      tick align=inside,
      nodes near coords,
      point meta = y,
      nodes near coords align={vertical},
      width=0.99\linewidth,
      height=2.5cm,
      enlarge x limits=0.17,
      legend pos=north west,
  ]

    \addlegendimage{ybar, area legend, pattern={north east lines}, pattern color=black}
    \addlegendentry{FunCoder}
    
    \addlegendimage{ybar, area legend, fill=black}
    \addlegendentry{PaT}

    \addplot[bar shift = -7pt, font=\tiny, fill=BLUE, pattern={north east lines}, pattern color=BLUE, nodes near coords] coordinates {(Easy,6.8)};
    \addplot[bar shift = -7pt, font=\tiny, fill=GREEN, pattern={north east lines}, pattern color=GREEN, nodes near coords={6.0}] coordinates {(Mid,6.0)};
    \addplot[bar shift = -7pt, font=\tiny, fill=YELLOW, pattern={north east lines}, pattern color=YELLOW, nodes near coords] coordinates {(Hard,19.5)};
    \addplot[bar shift = -7pt, font=\tiny, fill=RED, pattern={north east lines}, pattern color=RED, nodes near coords] coordinates {(Expert,25.4)};
    \addplot[bar shift = 7pt, font=\tiny, fill=BLUE, nodes near coords] coordinates {(Easy,24.6)};
    \addplot[bar shift = 7pt, font=\tiny, fill=GREEN, nodes near coords] coordinates {(Mid,38.6)};
    \addplot[bar shift = 7pt, font=\tiny, fill=YELLOW, nodes near coords] coordinates {(Hard,44.8)};
    \addplot[bar shift = 7pt, font=\tiny, fill=RED, nodes near coords] coordinates {(Expert,71.2)};
  \end{axis}
\end{tikzpicture}
    \vspace{-0.7cm}
    \caption{Decomposition rate}
    \vspace{-0.25cm}
    \label{fig:xCodeEval_decomposition}
    \end{subfigure}
    \hspace{0.6cm}
    \begin{subfigure}[!t]{0.272\textwidth}
    % \hspace*{-0.6cm}
    \begin{tikzpicture} % TikZ 그림 시작
    \begin{axis}[ % PGFPlots 축 설정 시작
        % 축 라벨
        xlabel={Cost ($10^{-2}$\$)},
        ylabel={Proportion (\%)},
        % 라벨 및 틱 라벨 폰트 크기
        label style={font=\scriptsize},
        ylabel style={yshift=-9pt},
        xlabel style={yshift=5pt},
        tick label style={font=\scriptsize},
        tick style={draw=none},
        legend style={font=\scriptsize},
        xmin=-100, xmax=10100, % 데이터 범위에 맞게 조정 (예: 50보다 작고, 2100보다 큰 값)
        ymin=0.0, ymax=100.0,   
        % 로그 스케일에 맞는 틱 라벨 설정 (예: 10, 100, 1000 등)
        xtick={0, 2000, 4000, 6000, 8000, 10000},
        xticklabels={0, 2, 4, 6, 8, 10},
        tick align=inside,
        scale only axis,
        grid=both, 
        major grid style={gray!50!white, line width=0.5pt}, 
        minor grid style={gray!20!white, line width=0.2pt}, 
        ymajorgrids, 
        xmajorgrids, 
        % 그래프 크기
        width=0.99\linewidth, 
        height=2.5cm,
        scaled x ticks = false,
    ]

    \addplot[
        color=BLUE,
        mark=diamond*, 
        line width=1pt,
        solid, 
        nodes near coords,
        point meta=explicit symbolic,
    ]
    coordinates { 
        (500, 21.6)
        (1500, 56.9)
        (2500, 12.3)
        (3500, 2.5)
        (4500, 2.5)
        (5500, 2.5)
        (6500, 1.0)
        (7500, 0.0)
        (8500, 0.5)
        (9500, 0.5)
    };
    \addlegendentry{Easy}; 

    \addplot[
        color=GREEN,
        mark=*, 
        line width=1pt,
        solid, 
        nodes near coords,
        point meta=explicit symbolic,
    ]
    coordinates {
        (500, 5.2)
        (1500, 45.5)
        (2500, 36.4)
        (3500, 6.5)
        (4500, 3.9)
        (5500, 0.0)
        (6500, 2.6)
        (7500, 0.0)
        (8500, 0.0)
        (9500, 0.0)
    };
    \addlegendentry{Mid}; 

    \addplot[
        color=YELLOW,
        mark=square*, 
        line width=1pt,
        solid, 
        nodes near coords,
        point meta=explicit symbolic,
    ]
    coordinates {
        (500, 5.2)
        (1500, 33.8)
        (2500, 22.1)
        (3500, 13.0)
        (4500, 6.5)
        (5500, 7.8)
        (6500, 7.8)
        (7500, 1.3)
        (8500, 0.0)
        (9500, 2.6)
    };
    \addlegendentry{Hard}; 

    \addplot[
        color=RED,
        mark=triangle*, 
        line width=1pt,
        solid, 
        nodes near coords, 
        point meta=explicit symbolic, 
    ]
    coordinates {
        (500, 8.2)
        (1500, 22.7)
        (2500, 22.7)
        (3500, 13.4)
        (4500, 7.2)
        (5500, 8.2)
        (6500, 7.2)
        (7500, 6.2)
        (8500, 2.1)
        (9500, 2.1)
    };
    \addlegendentry{Expert};

    % Easy 평균
    \addplot[BLUE, dashed] coordinates {(1967.5,0) (1967.5,100)};
    % Mid 평균
    \addplot[GREEN, dashed] coordinates {(2804.6,0) (2804.6,100)};
    % Hard 평균
    \addplot[YELLOW, dashed] coordinates {(4499.4,0) (4499.4,100)};
    % Expert 평균
    \addplot[RED, dashed] coordinates {(5717.0,0) (5717.0,100)};
    \end{axis}
\end{tikzpicture}
    \vspace{-0.7cm}
    \caption{Cost (FunCoder)}
    \vspace{-0.25cm}
    \label{fig:xCodeEval_FunCoder}
    \end{subfigure}
    \hspace{0.6cm}
    \begin{subfigure}[!t]{0.272\textwidth}
    % \hspace*{-0.6cm}
    \begin{tikzpicture} % TikZ 그림 시작
    \begin{axis}[ % PGFPlots 축 설정 시작
        % 축 라벨
        xlabel={Cost ($10^{-2}$\$)},
        ylabel={Proportion (\%)},
        % 라벨 및 틱 라벨 폰트 크기
        label style={font=\scriptsize},
        ylabel style={yshift=-9pt},
        xlabel style={yshift=5pt},
        tick label style={font=\scriptsize},
        tick style={draw=none},
        legend style={font=\scriptsize},
        xmin=-100, xmax=10100, % 데이터 범위에 맞게 조정 (예: 50보다 작고, 2100보다 큰 값)
        ymin=0.0, ymax=100.0,   
        % 로그 스케일에 맞는 틱 라벨 설정 (예: 10, 100, 1000 등)
        xtick={0, 2000, 4000, 6000, 8000, 10000},
        xticklabels={0, 2, 4, 6, 8, 10},
        tick align=inside,
        scale only axis,
        grid=both, 
        major grid style={gray!50!white, line width=0.5pt}, 
        minor grid style={gray!20!white, line width=0.2pt}, 
        % ymajorgrids, 
        % xmajorgrids, 
        % 그래프 크기
        width=0.99\linewidth, 
        height=2.5cm,
        scaled x ticks = false,
    ]

    \addplot[
        color=BLUE,
        mark=diamond*, 
        line width=1pt,
        solid, 
        nodes near coords,
        point meta=explicit symbolic,
    ]
    coordinates { 
        (500, 68.8)
        (1500, 16.8)
        (2500, 4.0)
        (3500, 4.0)
        (4500, 2.5)
        (5500, 2.0)
        (6500, 0.5)
        (7500, 1.0)
        (8500, 0.0)
        (9500, 0.5)
    };
    \addlegendentry{Easy}; 

    \addplot[
        color=GREEN,
        mark=*, 
        line width=1pt,
        solid, 
        nodes near coords,
        point meta=explicit symbolic,
    ]
    coordinates {
        (500, 48.7)
        (1500, 18.4)
        (2500, 11.8)
        (3500, 7.9)
        (4500, 5.3)
        (5500, 3.9)
        (6500, 2.6)
        (7500, 0.0)
        (8500, 0.0)
        (9500, 1.3)
    };
    \addlegendentry{Mid}; 

    \addplot[
        color=YELLOW,
        mark=square*, 
        line width=1pt,
        solid, 
        nodes near coords,
        point meta=explicit symbolic,
    ]
    coordinates {
        (500, 33.3)
        (1500, 29.5)
        (2500, 7.7)
        (3500, 12.8)
        (4500, 6.4)
        (5500, 5.1)
        (6500, 0.0)
        (7500, 1.3)
        (8500, 2.6)
        (9500, 1.3)
    };
    \addlegendentry{Hard}; 

    \addplot[
        color=RED,
        mark=triangle*, 
        line width=1pt,
        solid, 
        nodes near coords, 
        point meta=explicit symbolic, 
    ]
    coordinates {
        (500, 15.8)
        (1500, 19.7)
        (2500, 9.2)
        (3500, 13.2)
        (4500, 7.9)
        (5500, 6.6)
        (6500, 13.2)
        (7500, 5.3)
        (8500, 6.6)
        (9500, 2.6)
    };
    \addlegendentry{Expert};
    % Easy 평균
    \addplot[BLUE, dashed] coordinates {(1438.2,0) (1438.2,100)};
    % Mid 평균
    \addplot[GREEN, dashed] coordinates {(2670.2,0) (2670.2,100)};
    % Hard 평균
    \addplot[YELLOW, dashed] coordinates {(4347.7,0) (4347.7,100)};
    % Expert 평균
    \addplot[RED, dashed] coordinates {(9482.3,0) (9482.3,100)};
    \end{axis}
\end{tikzpicture}
    \vspace{-0.7cm}
    \caption{Cost (PaT)}
    \vspace{-0.25cm}
    \label{fig:xCodeEval_PaT}
    \end{subfigure}
\caption{\textbf{Adaptive decomposition probability and cost analysis for Qwen3$_\text{4B}$ on xCodeEval.} (a) Decomposition rate of FunCoder and PaT by problem difficulty. Per-difficulty cost distribution (solid line) and average cost (vertical dashed line) on (b) FunCoder and (c) PaT.}
\label{fig:xCodeEval_analysis}
\vspace{-0.5cm}
\end{figure*}
% Replanning 그림
% detail 한 feedback loop
\begin{figure}[!ht]
\centering
    \begin{subfigure}{0.45\textwidth}
        \centering
        \input{prompts/decomposition_FunCoder}
        \vspace{0.1cm}
        \caption{Decomposition prompt for FunCoder}
        \label{fig:decomposition_funcoder}
    \end{subfigure}
    \vspace{0.1cm}
    \begin{subfigure}{0.45\textwidth}
        \centering
        \input{prompts/decomposition_PaT}
        \vspace{0.1cm}
        \caption{Planning prompt for PaT}
        \label{fig:decomposition_adc}
    \end{subfigure}
    \vspace{-0.3cm}
    \caption{\textbf{Comparison of planning prompts.} Excerpts from planning prompts for (a) FunCoder and (b) PaT.}
    \label{fig:decomposition_prompt}
    \vspace{-0.5cm}
\end{figure}
\subsubsection{Challenging Benchmark}
\label{sec:complex}
On the challenging xCodeEval benchmark, PaT again achieves higher performance than all baselines across all model variants.
However, we observe an interesting cost dynamic: for smaller models like Qwen3$_\text{4B}$ and Llama3.1$_\text{8B}$, PaT incurs a higher cost than FunCoder.
This is not a sign of inefficiency but a direct consequence of PaT's adaptive strategy.
Less capable models fail more frequently on xCodeEval's difficult problems, and PaT correctly interprets these failures as signals to invest in decomposition.
This strategic escalation, while costly, is precisely what allows these smaller models to overcome challenges where rigid policies like FunCoder's fall short.

For a deeper analysis, we examine the behavior of Qwen3$_\text{4B}$ in Figure~\ref{fig:xCodeEval_analysis}, which reveals a fundamental limitation of pre-emptive planning.
While both PaT and FunCoder increase decomposition for harder problems, PaT exhibits a much more dynamic response compared to FunCoder's more restrained increase.
Remarkably, despite this aggressive decomposition, the cost difference remains marginal because FunCoder incurs a planning overhead on 100\% of problems, whereas PaT invokes the planner exclusively on failed instances.
This reinforces a key insight: planning is a high-cost resource that must be allocated reactively, relying on the signal of failure to distinguish necessary investment from wasteful overhead.

This behavioral divergence stems from the distinct prompt strategies shown in Figure~\ref{fig:decomposition_prompt}.
While FunCoder relies on the model's subjective assessment of complexity, PaT leverages a grounded feedback signal derived from trial failure to issue an imperative command.
This explicit directive proves far more effective at triggering necessary decomposition than a conditional instruction, ensuring the model does not underestimate the task's difficulty.

\subsection{Heterogeneous Setting}
\label{sec:hybrid-exp}

Our discussion in Section~\ref{sec:heterogeneous} established the potential for enhanced cost-efficiency in a heterogeneous model configuration.
To empirically validate this, we conduct experiments on the foundational benchmarks, pairing a powerful planner (Qwen3$_\text{32B}$) with a series of smaller generator models.
The results, presented in Table~\ref{table:hetero_table}, provide strong support for our analysis.
Pairing a Qwen3$_\text{8B}$ generator with a Qwen3$_\text{32B}$ planner achieves competitive performance (an average Pass@1 of $87.39\%$), with a <1\% gap to the homogeneous Qwen3$_\text{32B}$, while reducing the relative cost to just $0.31$.

The superior cost-benefit trade-off of the heterogeneous approach is visualized in Figure~\ref{fig:hetero_plot}.
In this graph, the slope of the curve represents the performance return on additional cost.
The heterogeneous model configurations exhibit a significantly steeper slope, demonstrating that upgrading only the planner is a highly capital-efficient strategy, yielding substantial performance gains for a marginal increase in cost.
This confirms our central hypothesis: because PaT invokes the planner infrequently, reserving a powerful model for this critical but rare task is the most cost-effective way to enhance the overall system's capability.

\begin{table}[!t]
\centering
\small
\begin{tabular}{l|lrr}
    \toprule
    \textbf{Generator} & \textbf{Planner} & \textbf{Avg.} & \textbf{Cost} \\
    \midrule 
    Qwen3$_\text{32B}$ & \cellcolor{col!25} Qwen3$_\text{32B}$ & \cellcolor{col!25} 88.37 & \cellcolor{col!25} 1.00 \\ 
    \cmidrule(l){1-4}
    \multirow{2}{*}{Qwen3$_\text{14B}$} & \cellcolor{col!25} Qwen3$_\text{14B}$ & \cellcolor{col!25} 86.18 & \cellcolor{col!25} 0.47 \\
    & \cellcolor{col2!25} Qwen3$_\text{32B}$ & \cellcolor{col2!25} 87.53 & \cellcolor{col2!25} 0.49 \\
    \cmidrule(l){1-4}
    \multirow{2}{*}{Qwen3$_\text{8B}$} & \cellcolor{col!25} Qwen3$_\text{8B}$ & \cellcolor{col!25} 85.58 & \cellcolor{col!25} 0.25 \\
    & \cellcolor{col2!25} Qwen3$_\text{32B}$ & \cellcolor{col2!25} 87.39 & \cellcolor{col2!25} 0.31 \\
    \cmidrule(l){1-4}
    \multirow{2}{*}{Qwen3$_\text{4B}$} &\cellcolor{col!25} Qwen3$_\text{4B}$ & \cellcolor{col!25} 83.13 & \cellcolor{col!25} 0.14\\ 
    & \cellcolor{col2!25} Qwen3$_\text{32B}$ & \cellcolor{col2!25} 84.78 & \cellcolor{col2!25} 0.18 \\
    \bottomrule
\end{tabular}
\caption{\textbf{Performance (Pass@1) and relative cost results} for homogeneous and heterogeneous PaT configurations. The cost is normalized relative to the homogeneous Qwen3$_\text{32B}$ model, which is set to $1.00$.}
\vspace{-0.2cm}
\label{table:hetero_table}
\end{table}
\begin{figure}[!t]
\centering
\begin{tikzpicture} % TikZ 그림 시작
    \begin{axis}[ % PGFPlots 축 설정 시작
        % 축 라벨
        xlabel={Cost},
        ylabel={Pass@1 (\%)}, 
        ylabel style={yshift=-5pt},
        xlabel style={yshift=5pt},
        % 라벨 및 틱 라벨 폰트 크기
        label style={font=\footnotesize},
        tick label style={font=\scriptsize},
        % --- 축 범위 설정 ---
        xmin=0.0, xmax=1.1, 
        ymin=82, ymax=90, 
        xtick={0.2, 0.6, 1.0}, 
        ytick={82, 86, 90},
        % 그리드 설정
        grid=both, 
        major grid style={gray!50!white, line width=0.5pt}, 
        minor grid style={gray!20!white, line width=0.2pt}, 
        ymajorgrids, 
        xmajorgrids, 
        % 그래프 크기
        width=0.45\textwidth, % <--- 전체 텍스트 너비에 맞게 조정 (0.7은 예시)
        height=5cm,
        % 범례 스타일 설정
        legend style={at={(0.95, 0.05)}, anchor=south east, font=\footnotesize},
    ]
        \addplot[
            color=col,
            mark=square*, 
            line width=1pt,
            solid, 
            nodes near coords, 
            point meta=explicit symbolic, 
            every node near coord/.append style={font=\scriptsize,xshift=0.0em,yshift=-1.5em} 
        ]
        coordinates {
            (0.14, 83.13) [\;\;\;\;\;\;\;\;\;(4B, 4B)] 
            (0.25, 85.58) [\;\;\;\;\;\;\;\;\;\;\;(8B, 8B)] 
            (0.47, 86.18) [\;\;\;\;\;(14B, 14B)] 
            (1.00, 88.37) [\!\!(32B, 32B)]
        };
        \addlegendentry{Homogeneous};
    
        \addplot[
            color=col2,
            mark=star,
            mark indices={2, 4, 6},
            dashed,
            line width=1pt,
            nodes near coords,
            point meta=explicit symbolic,
            every node near coord/.append style={font=\scriptsize,xshift=0.0em,yshift=-0.1em}
        ]
        coordinates {
            (0.14, 83.13)
            (0.18, 84.78) [\!\!\!\!\!\!\!\!\!\!\!\!\!\!(4B, 32B)]
            
            (0.25, 85.58)
            (0.31, 87.39) [\!\!\!\!(8B, 32B)]
            
            (0.47, 86.18)
            (0.49, 87.53) [\;\;\;(14B, 32B)]
        };
        \addlegendentry{Heterogeneous}; % 캡션이 명확하므로 범례 텍스트를 간결하게 수정
        % \node[font=\scriptsize, color=col2, anchor=north west, xshift=1pt, yshift=-1pt] at (axis cs: 0.18, 84.78) {(32B, 4B)};
    \end{axis}
\end{tikzpicture}
\vspace{-0.3cm}
\caption{\textbf{Trade-off curve for heterogeneous configurations.} Corresponding to Table~\ref{table:hetero_table}, this plot visualizes the Pass@1 performance vs. relative cost. Labels denote (Generator, Planner) sizes.
Employing a powerful Qwen3$_\text{32B}$ planner (dashed line) yields a significant performance gain for a marginal increase in cost compared to Homogeneous configurations (solid line).}
\vspace{-0.5cm}
\label{fig:hetero_plot}
\end{figure}
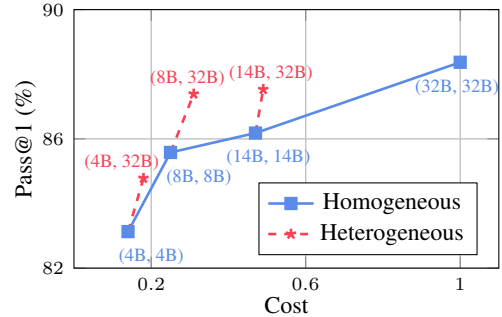

\section{Conclusion}
\label{sec:conclusion}
In this work, we introduced Planning-after-Trial (PaT), an adaptive policy that addresses the critical challenge of high inference costs in LLM-based code generation.
By inverting the conventional Planning-before-Trial (PbT) policy, PaT avoids unnecessary planning overhead on simple problems and strategically allocates computational resources only when a direct attempt fails.
We demonstrated that this adaptive nature is uniquely synergistic with a heterogeneous configuration, allowing for the decoupling of planning and generation costs.
Comprehensive experiments confirmed that PaT outperforms existing state-of-the-art baselines in homogeneous settings across diverse models and benchmarks.
Furthermore, we validated that our heterogeneous model configuration establishes a superior cost-performance frontier compared to homogeneous setups.
By enabling principled test-time computation scaling, PaT provides a practical and effective framework for building more scalable and cost-efficient code generation systems.

\section*{Limitations}
Relying on self-generated test cases inherently introduces noise and potential inefficiency arising from unnecessary planning for correctly solved problems.
While integrating emerging high-fidelity generation methods \citep{wang2025testeval, prasad2025learning} offers synergistic potential, achieving a perfectly noise-free oracle remains unlikely; thus, robust mechanisms like our plateau heuristic remain essential for practical deployment.
Regarding computational resources, our heterogeneous configuration entails a higher static memory footprint.
However, this design aligns with efficient architectures like Mixture-of-Experts (MoE) \citep{jiang2024mixtral}, prioritizing the significant reduction of average dynamic inference cost and user-facing latency for the majority of tasks.

\section*{Acknowledgements}
This work was supported by Institute of Information \& Communications Technology Planning \& Evaluation (IITP) grants funded by the Korea government (MSIT) (RS-2019-II191906, Artificial Intelligence Graduate School Program (POSTECH); RS-2024-00457882, AI Research Hub Project; IITP-2026-RS-2024-00437866, Information Technology Research Center (ITRC); RS-2026-25511821, Development of Personalized Media Service Recommendation and Generative Technology). It was also supported by the MSIT under the Global Research Support Program in the Digital Field (RS-2024-00436680) supervised by the IITP. This project is supported by Microsoft Research Asia.

\section*{AI writing assistance}
We utilized large language models solely for the purpose of refining the clarity, grammar, and style of the manuscript under the scope of ``Assistance purely with the language of the paper.''
All scientific claims, experimental designs, and empirical results presented in this paper are the original work of the authors.

\bibliography{acl2026}

\appendix

% \section{Pseudo Code}
% \label{sec:app_pseudo}
% \input{iclr2026/algorithms/algorithm}
% \onecolumn
% \newpage

\section{Theoretical Analysis}
\label{sec:app_theory}

In this section, we provide a formal theoretical analysis supporting the cost-efficiency of the Heterogeneous Model Configuration in Section~\ref{sec:heterogeneous}.
We first establish the notations and cost model, then prove the existence of an efficient heterogeneous policy (Theorem~\ref{thm:exist_ps}).
We subsequently extend our analysis to consider the asymptotic cost for problems of large complexity $k$ and derive the optimal cost for the generator sLM under scaling laws.

\paragraph{Preliminaries and Notations.}
We denote the problem complexity by $k > 0$.
We consider a set of models $M \in \{L, s\}$, where $L$ represents a large, capable model (used as the Planner) and $s$ represents a smaller, cost-efficient model (used as the Generator).
For each model $M$, let $p_M$ denote its problem-solving capability (the maximum complexity it can solve in a single attempt) and $c_M$ denote its unit inference cost.
Consistent with the heterogeneous setting, we assume $0 < p_s < p_L$ and $0 < c_s < c_L$.

Based on these notations, we define the core assumptions that characterize our cost model.

\begin{assumption}[Generation cost] 
\label{ass:gen_cost}
A generation call by $M_G$ on a problem of complexity $k$ incurs
\begin{align}
\text{Cost}_\text{Generation}=
\begin{cases}
k\,c_{M_G}, \!\!\!\!& k \le p_{M_G}  \;\text{(success)},\\[2pt]
p_{M_G}\,c_{M_G}, \!\!\!\!& k > p_{M_G} \;\text{(failed)}.
\end{cases}
\end{align}
\end{assumption}

\begin{assumption}[Planning cost]
\label{ass:plan_cost}
A Planner-produced $n$-way plan maps a problem of size $k$ to $n$ independent subproblems of size $\frac{k}{n}$ and incurs a fixed per-subproblem overhead $D_{M_P}$, so that
\begin{align}
\text{Cost}_\text{Planning} \;=\; nD_{M_P}.
\end{align}
\end{assumption}

To formally connect a model's capability $p_M$ and its unit cost $c_M$, we adopt a standard assumption from the literature on neural scaling laws.
\begin{assumption}[Scaling law]
\label{ass:scaling_law}
We assume that a model's capability $p_M$ and its unit cost $c_M$ are related by a power-law scaling relation, consistent with prior work \citep{kaplan2020scaling}:
\begin{align}
\label{eq:scaling-law}
p_M=\alpha\,c_M^{\beta},\qquad \alpha>0,\ 1 \ge \beta> 0.
\end{align}
The constraint $0 < \beta \le 1$ reflects the principle of diminishing returns.
\end{assumption}

With these definitions established, we now formally state the condition under which the heterogeneous policy outperforms the homogeneous one.

\begin{theorem}[Existence of an efficient sLM capability]
\label{thm:exist_ps}
Let $k \sim \mathrm{Unif}(0,p_L]$.
Under Assumptions~\ref{ass:gen_cost}, \ref{ass:plan_cost}, and \ref{ass:scaling_law}, if the total planning overhead satisfies
\begin{align}
n D_L \;<\; \Big(\tfrac{1}{2}-\tfrac{1}{n^{2}}\Big)\, p_L c_L,
\end{align}
then there exists $p_s \in (0,p_L)$ for which PaT with a Heterogeneous Model Configuration has strictly lower expected cost than an LLM-only policy.
\end{theorem}

\begin{proof}
First, for the homogeneous LLM-only policy, the cost for any problem $k \le p_L$ is $kc_L$.
The expected cost is therefore:
\begin{align}
    &\mathbb{E}[\text{Cost}_\text{Homogeneous}] \notag\\&= \frac{1}{p_L} \int_0^{p_L} k c_L \;dk = \frac{1}{p_L} \frac{p_L^2c_L}{2} = \frac{p_Lc_L}{2}
\end{align}

Next, for the Heterogeneous policy, we consider two cases based on the capability of the small model, $p_s$.
If $p_s\ge k$, the generator $s$ succeeds, incurring a cost of $kc_s$.
Else, \ie $p_s < k$, the generator $s$ fails (cost $p_s c_s$), the planner $L$ is invoked (cost $nD_L$), and then the generator $s$ solves the $n$ subproblems of size $\frac{k}{n}$. 
Consider $p_s \ge \frac{p_L}{n}$, we can calculate the cost in this case $p_s c_s + nD_L + k c_s$.
The expected cost is the sum of the integrals over these two ranges, divided by $p_L$:
\begin{align}
    &\mathbb{E}[\text{Cost}_\text{Heterogeneous}] \notag\\
    &= \frac{1}{p_L} \biggl(\int_0^{p_s} \!\!k c_s \;dk+\!\int_{p_s}^{p_L} \!\!p_s c_s +\! nD_L +\! k c_s \; dk \biggr)\\
    &= \frac{1}{p_L} \biggl(\frac{p_s^2c_s}{2} + p_sc_s (p_L-p_s)+ nD_L (p_L-p_s)\notag 
    \\& \;\;\;\;\;\;\;\;\;\;\;\;\;+\frac{(p_L^2-p_s^2)c_s}{2} \biggr)\\
    &= \frac{p_{L} c_{s}}{2} + \frac{p_{L}-p_{s}}{p_{L}} \left(p_s c_s + nD_L\right) \label{eq:hetero1}
\end{align}

To analyze a concrete scenario, let's assume we can choose an sLM such that its capability is a fraction of the LLM's, \ie $p_s = \frac{p_L}{n}$.
Using the scaling law from Assumption~\ref{ass:scaling_law} ($p_M = \alpha c_M^\beta$), we can relate the costs: $c_s = \left(\frac{p_s}{\alpha}\right)^{1/\beta}=\left(\frac{p_L}{n\alpha}\right)^{1/\beta}=n^{-1/\beta}c_L$.

Substituting these into the heterogeneous cost \eqref{eq:hetero1}:
\begin{align}
    &\mathbb{E}[\text{Cost}_\text{Heterogeneous}] \\
    &= \frac{p_{L} c_{s}}{2} + \frac{p_{L}-p_{s}}{p_{L}} \left(p_s c_s + nD_L\right) \\
    &=\frac{p_Ln^{-1/\beta}c_L}{2} + \frac{n-1}{n}\left(\frac{p_L}{n}n^{-1/\beta}c_L+nD_L\right)\\
    &=\left(\frac{1}{2}+\frac{n-1}{n^2}\right)n^{-1/\beta}p_Lc_L + (n-1)D_L
\end{align}

For the heterogeneous policy to be strictly more efficient, the difference must be positive:
\begin{align}
    0 &> \mathbb{E}[\text{Cost}_{\text{Heterogeneous}} - \text{Cost}_{\text{Homogeneous}}] \\
    &=\!(n-1)D_L \!+ \!\left(\!\frac{1}{2} \!+\! \frac{n-1}{n^2}\!\right)n^{-\frac{1}{\beta}}p_Lc_L \!-\! \frac{p_{L}c_{L}}{2}.
\end{align}

Rearranging to solve for the planning overhead $nD$ gives the condition stated in the theorem:
\begin{align}
    nD_L &< \frac{n}{n-1}\!\left(\!\frac{p_{L}c_{L}}{2} \!-\! \left(\!\frac{1}{2} \!+\! \frac{n-1}{n^2}\!\right)n^{-\frac{1}{\beta}}p_L c_L\!\right)\! \\&= \left(\frac{n-n^{1-\frac{1}{\beta}}}{2(n-1)}-n^{-1-\frac{1}{\beta}}\right) p_{L} c_{L}.
\end{align}

As a specific intuitive case, if we consider linear scaling where $\beta = 1$, the condition simplifies to:
\begin{align}
    nD_L < \left(\frac{1}{2} - \frac{1}{n^2}\right) p_L c_L
\end{align}

This shows that as long as the total planning cost is less than the savings achieved by the heterogeneous configuration, a more efficient sLM exists.
\end{proof}

\begin{theorem}[Asymptotic efficiency of the heterogeneous configuration]
\label{theorem:hybrid}
For any sufficiently complex task $k$, the Heterogeneous strategy is asymptotically more cost-efficient than the Homogeneous strategy, provided that the cost of decomposition satisfies the following condition:
\begin{align}
D_L < \frac{c_{L} - c_{s}}{\frac{1}{p_{s}} - \frac{1}{p_{L}}}.
\end{align}
\end{theorem}
\begin{proof}
We prove the theorem by analyzing the asymptotic cost of each strategy for a large problem of complexity $k$.
Let $h_M = \lceil k/p_{M}\rceil$ be the number of recursive division levels for a model $M$.
The total number of division operations is $1+ n + n^2 +... + n^{h_M-1} = \frac{n^{h_M}-1}{n-1}$.
For a large $k$, we can approximate it as $\frac{n^{h_M}-1}{n-1} \approx \frac{k}{p_M(n-1)}$. 

The total cost for a model $M$ is the sum of three components: the cumulative cost of failures at each division step $\frac{k}{p_M(n-1)}p_Mc_M$, the cumulative cost of decomposition $\frac{k}{p_M(n-1)}nD_L$, and the final cost of conquering the sub-problems $kc_M$.
Summing these components gives the total asymptotic cost: 
\begin{align}
\frac{nk}{n-1}\left(c_M + \frac{D_L}{p_M}\right).
\end{align}

For the heterogeneous setting to be more cost-efficient than the homogeneous one, it requires:
\begin{align}
\frac{nk}{n-1}\left(c_{s} + \frac{D_L}{p_{s}}\right) < \frac{nk}{n-1}\left(c_{L} + \frac{D_L}{p_{L}}\right).
\end{align}

The $\frac{nk}{n-1}$ term cancels.
Rearranging the remaining terms to solve for $D_L$ yields the condition stated in the theorem:
\begin{align}
D_L < \frac{c_{L} - c_{s}}{\frac{1}{p_{s}} - \frac{1}{p_{L}}}.
\end{align}

\end{proof}

\begin{table*}[!t]
\centering
\small
\begin{tabular}{L{1.5cm}L{1.5cm}R{0.8cm}R{0.8cm}R{0.8cm}R{0.8cm}R{0.8cm}R{1.0cm}R{0.8cm}}
\toprule
\multirow{2}{*}{\textbf{Generator}} & \multirow{2}{*}{\textbf{Planner}} & \multicolumn{2}{c}{\textbf{HumanEval}} & \multicolumn{2}{c}{\textbf{MBPP}} & \multirow{2}{*}{\textbf{Avg.}} & \multirow{2}{*}{$\Delta$ \textbf{Avg.}}& \multirow{2}{*}{\textbf{cost}}
\\
\cmidrule(lr){3-4} \cmidrule(lr){5-6}
& & \multicolumn{1}{c}{base} & \multicolumn{1}{c}{plus} & \multicolumn{1}{c}{base} & \multicolumn{1}{c}{plus} & & & 
\\
\midrule 
Qwen3$_\text{32B}$ & Qwen3$_\text{32B}$ & 93.90 & 84.15 & 88.57 & 86.86 & 88.37 & - & 1.00 \\
\cmidrule(lr){1-9}
\multirow{2}{*}{Qwen3$_\text{14B}$} & Qwen3$_\text{14B}$ & 91.46 & 83.54 & 85.14 & 82.29 & 86.18 & - 2.19 & 0.47 \\
 & Qwen3$_\text{32B}$ & 93.29 & 85.98 & 86.29 & 84.57 & 87.53 & - 0.84 & 0.49 \\
\cmidrule(lr){1-9}
\multirow{2}{*}{Qwen3$_\text{8B}$} & Qwen3$_\text{8B}$ & 90.85 & 82.32 & 85.14 & 84.00 & 85.58 & - 2.79 & 0.25 \\
 & Qwen3$_\text{32B}$ & 93.90 & 85.37 & 86.29 & 84.00 & 87.39 & - 0.98 & 0.31 \\
\cmidrule(lr){1-9}
\multirow{2}{*}{Qwen3$_\text{4B}$} & Qwen3$_\text{4B}$ & 89.63 & 82.32 & 82.29 & 78.29 & 83.13 & - 5.24 & 0.14 \\
& Qwen3$_\text{32B}$ & 92.07 & 84.76 & 82.29 & 80.00 & 84.78 & - 3.40 & 0.18 \\
\bottomrule
\end{tabular}
\caption{\textbf{Full results on heterogeneous setting.} We report the average cost relative to Qwen3$_\text{32B}$ is normalized to 1.00.}
\label{table:Hybrid}
\end{table*}

This theorem provides a theoretical basis for the efficiency of heterogeneous model configuration.
It indicates that the heterogeneous configuration becomes more cost-efficient when the cost of decomposition ($D_L$) is less than the savings generated by executing the sub-problems with a more cost-effective model ($s$).
This provides a formal rationale for using the heterogeneous configuration and shows that the PaT policy is a structured approach to allocating computational resources.

\begin{theorem}[Optimal Generator cost under scaling laws]
\label{theorem:optimal_slm}
The cost of the optimal small model, $c_{s}^*$, that minimizes the asymptotic cost of the heterogeneous configuration is given by the following closed-form solution:
\begin{align}
c_{s}^* = \min\left\{\left(\frac{\beta \cdot D_L}{\alpha}\right)^\frac{1}{\beta+1}, c_L\right\}
\end{align}
\end{theorem}
\begin{proof}
The asymptotic cost of the heterogeneous configuration is proportional to the coefficient $c_{s} + \frac{D_L}{p_{s}}$.
Substituting the scaling law gives $c_{s} + \frac{D_L}{\alpha c_{s}^\beta}$. 
To find the minimum, we take the derivative with respect to $c_{s}$ and set it to zero:
\begin{align}
    1-\frac{\beta D_L}{\alpha}c_{s}^{-\beta-1} = 0.
\end{align}
Solving for $c_{s}$ yields the unconstrained optimum.
The final solution is capped at $c_{L}$ to respect the problem's practical constraints.
\end{proof}

This theorem provides a practical, closed-form solution for the cost of the optimal generator model in the asymptotic regime.
The result serves as a powerful heuristic to approximately estimate the most cost-effective smaller model to use in real-world heterogeneous configurations.

\paragraph{Empirical Validation.}
To validate the practical utility of Theorem~\ref{theorem:optimal_slm}, we applied the formula to our HumanEval experimental data using Qwen3 models.
We modeled model capability $p$ against normalized input token costs $c \in \{0.11, 0.18, 0.35, 0.70\}$ using the scaling law $p = \alpha c^\beta$.
Fitting this to our observed data yielded $\alpha \approx 1722.2$ and $\beta \approx 0.12$, with an average planning cost $D_L \approx 1270.73$ derived from the heterogeneous (4B+32B) experiments.
Substituting these parameters into Theorem~\ref{theorem:optimal_slm} yields a theoretical optimal cost $c_s^* \approx 0.114$.
This value is remarkably close to the actual cost of the 4B model ($c=0.11$).
While our main experiments (Figure~\ref{fig:hetero_plot}) suggest the 8B model offers a strong performance-cost balance, the 4B model is indeed the strictly cost-optimal generator as predicted.
This alignment confirms that despite simplified assumptions, our theoretical model successfully captures the underlying cost dynamics and serves as a practical guideline for selecting the initial sLM size for hyperparameter search.

\section{Additional Results}
\label{sec:app_additional_experiments}
\subsection{Full version of Table~\ref{table:hetero_table}}
Table~\ref{table:Hybrid} presents the full results for our heterogeneous setting experiments across all four foundational benchmarks (HumanEval, HumanEval+, MBPP, and MBPP+).
In this configuration, we pair a series of smaller generator models with a powerful, fixed planner (Qwen3$_\text{32B}$).
The table provides the detailed Pass@1 scores and relative costs for each combination, which were aggregated and visualized in Figure~\ref{fig:hetero_plot} in the main body of the paper.

\subsection{Wall-Clock Latency Analysis}
To complement the token-cost analysis in the main paper, we 
measure end-to-end wall-clock latency in an isolated 
single-GPU environment to avoid interference from shared 
resources. Table~\ref{table:latency} reports the mean 
per-problem latency (in seconds) on HumanEval using Qwen3 
models.

\begin{table}[!t]
\centering
\small
\begin{tabular}{lrrr}
\toprule
\textbf{Model} & Standard & FunCoder & PaT \\
\midrule
Qwen3$_\text{4B}$ & 3.87 & 19.82 & 11.14 \\
Qwen3$_\text{32B}$ & 14.21 & 63.56 & 41.39 \\
\cmidrule(lr){1-4}
Qwen3$_{\text{4B}+\text{32B}}$ & - & - & 11.90 \\
\bottomrule
\end{tabular}
\caption{\textbf{Mean wall-clock latency per problem 
(seconds) on HumanEval.} Measured in an isolated 
environment using Qwen3 models. Qwen3$_{\text{4B}+\text{32B}}$ denotes the heterogeneous model configuration using a 4B generator and 32B planner.}
\label{table:latency}
\end{table}

PaT incurs roughly 3$\times$ overhead over Standard, while 
FunCoder incurs 5$\times$, confirming that PaT's 
adaptive policy reduces not only token cost but also 
real-world execution time. Notably, the heterogeneous 
configuration (4B+32B) achieves a latency of 11.90s, nearly 
identical to the 4B-only PaT (11.14s) and faster than the 
32B Standard baseline (14.21s). This demonstrates that 
reserving the larger model exclusively for planning keeps 
wall-clock time dominated by the cheaper generator.

\begin{table}[!t]
\centering
\small
\begin{tabular}{lrrrr}
\toprule
\textbf{Model} & Best-of-N & CodeT & FunCoder & PaT \\
\midrule
Qwen3$_\text{4B}$ & 0.38 & 0.03 & 0.70 & \textbf{1.84} \\
Qwen3$_\text{8B}$ & 0.08 & 0.23 & 0.73 & \textbf{1.99} \\
Qwen3$_\text{14B}$ & 0.50 & 0.28 & 0.64 & \textbf{1.63} \\
Qwen3$_\text{32B}$ & 0.91 & 0.33 & 0.54 & \textbf{1.23} \\
\cmidrule(lr){1-5}
Llama3.1$_\text{8B}$ & \textbf{2.10} & 1.06 & 0.90 & 1.89 \\
\cmidrule(lr){1-5}
DeepSeek-Coder\!\!\!\!\!\! & 0.04 & 0.11 & 0.51 & \textbf{0.91} \\
\cmidrule(lr){1-5}
Average & 0.67 & 0.34 & 0.67 & \textbf{1.58} \\
\bottomrule
\end{tabular}
\caption{\textbf{Compute ROI across models and methods 
on foundational benchmarks.} ROI is defined as 
$\Delta\text{Avg} / (\text{Cost} - 1)$, measuring 
Pass@1 improvement per unit of additional cost over 
Standard. Higher is better.}
\label{table:roi}
\end{table}

\subsection{Compute Return on Investment}
To provide a unified metric for the cost-performance 
trade-off, we define the compute return on investment (ROI) 
as the performance gain per unit of additional cost over 
Standard:
\begin{equation}
    \text{ROI} = \frac{\Delta\text{Avg}}{\text{Cost} - 1},
\end{equation}
where $\Delta\text{Avg}$ is the average Pass@1 improvement 
over Standard across benchmarks, and Cost is the normalized 
cost reported in Table~\ref{table:HumanEval} (Standard = 1). 
A higher ROI indicates a more efficient use of additional 
compute. Table~\ref{table:roi} reports the ROI for all 
methods.

PaT achieves an average ROI of 1.58, approximately 
2.4$\times$ higher than FunCoder (0.67) and 4.6$\times$ 
higher than CodeT (0.34). This confirms that PaT's adaptive 
policy consistently delivers superior returns on additional 
test-time investment across all model scales and families.

\section{Implementation Details}
\label{sec:app_implementation}
\paragraph{Generation.}
The trial phase of our PaT policy incorporates a Best-of-N strategy to maximize the chance of a direct success.
For each problem specification, the generator model $M_G$ produces 5 candidate solutions, using a temperature of 0.8 to encourage diverse outputs.
Each of these 5 candidates is then verified against the test set.
If any of the candidates passes all test cases, the process terminates successfully, and that solution is returned.
Only if all 5 candidates fail the verification step does the policy escalate to the planning phase.
The generation is retried up to 3 times if it fails to produce a parsable output.

\paragraph{Planning.}
If the trial fails, the planning phase is triggered.
The planner model $M_P$ is prompted to decompose the problem specification $x$ by rewriting the main function to call new unimplemented helper functions.
A single decomposition plan is generated with a temperature of 0.2.
From the resulting Python code block, we parse the new function signatures and docstrings, which become the specifications for the subproblems.
If the output is not a valid code block, this planning step is retried up to 3 times, and the maximum recursion depth is limited to 3 to prevent overly complex solutions.
Once all subproblems are recursively solved, their solutions are composed into a final program and verified against the original test suite for $x$.
If this composite solution still fails, PaT initiates a re-planning loop.
The planner is invoked again on the original problem $x$, but this time it is provided with the set of previously successful helper functions as additional context, enabling a more informed and cost-efficient re-planning attempt.

\paragraph{Test case generation \& verification.}
Our verification process requires a test suite $\mathcal{T}(x)$ for each problem $x$, which we construct in two stages.
First, we process any example test cases provided directly in the problem description.
Then, to ensure a comprehensive and robust evaluation, we augment the initial suite by prompting an LLM to generate additional test cases based on the problem description, a technique inspired by CodeT.
This test generation process is performed consistently with a temperature of $0.2$ and is retried up to $3$ times.

\section{Benchmark Details}
\label{sec:app_benchmark}
\paragraph{HumanEval} \citep{chen2021evaluating}
is a foundational dataset for evaluating the functional correctness of generated code.
It consists of 164 hand-written programming problems, each including a function signature, a detailed docstring, and a set of hidden unit tests for evaluation.
This benchmark is the standard for measuring the code generation capabilities of a model.

\paragraph{MBPP} \citep{chen2022program}
is a larger, crowd-sourced dataset.
A critical challenge with the original MBPP dataset is that the prompts contain the ground-truth test cases, which can cause label leakage, particularly for baselines that perform selection or refinement.
To ensure a fair comparison, we follow the setup of \cite{shinn2023reflexion}.
For our experiments, we use a representative subset of 175 problems sampled from the full dataset.

\paragraph{EvalPlus} \citep{evalplus}
ensures a rigorous and reliable evaluation by augmenting the standard test suites of both HumanEval and MBPP.
HumanEval and MBPP sometimes pass solutions that are functionally correct on the provided tests but fail on more subtle edge cases.
EvalPlus mitigates this risk by automatically generating a much larger and more comprehensive set of unit tests, providing a more robust measure of a candidate solution's true correctness.

\paragraph{xCodeEval} \citep{khan2023xcodeeval}
is more challenging, competition-level problems, which is sourced from the CodeForces platform.
For our experiments, we sample a subset of 500 problems, after first filtering out any instances with incomplete or invalid test cases.
A key feature of xCodeEval is its difficulty labels, which allow us to partition problems into four categories (Easy, Mid, Hard, and Expert) for a more fine-grained analysis of policy behavior.

\begin{table}[!ht]
\centering
\small
\begin{tabular}{lrr}
\toprule
\textbf{Model} & Input & Output \\
\midrule
Qwen3$_\text{4B}$ & 0.11 & 0.42 \\
Qwen3$_\text{8B}$ & 0.18 & 0.70 \\
Qwen3$_\text{14B}$ & 0.35 & 1.40 \\
Qwen3$_\text{32B}$ & 0.70 & 2.80 \\
\cmidrule(lr){1-3}
Llama3.1$_\text{8B}$ & 0.10 & 0.10 \\
\cmidrule(lr){1-3}
DeepSeek-Coder& 0.14 & 0.28 \\
\bottomrule
\end{tabular}
\caption{\textbf{Token pricing (USD per million tokens).} The values represent the cost for input and output, respectively.}
\label{table:LLM_cost}
\end{table}

\section{LLMs Details}
\label{sec:app_LLMs}
We selected a range of state-of-the-art open-source language models to evaluate our PaT policy across different architectures and scales.
All experiments were conducted using NVIDIA A 6000 GPUs, with models served via the vLLM framework at float16 precision.
Due to its memory size, the Qwen3$_\text{32B}$ model was run using 2-way tensor parallelism across two A6000 GPUs.
Table~\ref{table:LLM_cost} reports input and output token prices (USD per million tokens) collected from public provider listings\footnote{\url{https://artificialanalysis.ai} (accessed 2025-09-25).}.
Since there is no official price available for DeepSeek-Coder-V2-Lite-Instruct, we use the pricing policy of DeepSeek-Coder-V2 as a proxy.

\paragraph{Qwen3} \citep{yang2025qwen3}
As our primary model family, we use the Qwen 3 series in four sizes: 4B, 8B, 14B, and 32B.
For all experiments, we used the base, pre-trained versions of these models. Additionally, for all experiments with the Qwen3 model family, we prepended the /no\_think command to all prompts.

\paragraph{Llama-3.1} \citep{dubey2024llama}
To test the cross-family generalization of our approach, we use Llama-3.1-8B-Instruct. This is the instruction-tuned version of the Llama-3.1 8B model, optimized for following user commands and prompts.

\paragraph{DeepSeek-Coder-V2-Lite} \citep{zhu2024deepseek}
To evaluate on a model specifically fine-tuned for code generation, we use DeepSeek-Coder-V2-Lite-Instruct.
This is the instruction-tuned ``Lite'' version of the DeepSeek-Coder-V2 family, with approximately 16B parameters.

\begin{table*}[!ht]
\centering
\small
\begin{tabular}{l|ccccc}
\toprule
\textbf{Method} & \textbf{Standard} & \textbf{Best-of-N} & \textbf{CodeT} & \textbf{FunCoder} & \textbf{PaT}\\
\midrule
\rowcolor{col!20}\multicolumn{6}{c}{\textbf{Generation}}\\
\midrule
Samples (N)    &  1  & 5 & 11 & 1 + 10 & 5 \\
Temperature & 0.3 & 0.8 & 0.8 & 0.8 & 0.8 \\
Retries & 3 & 3 & 3 & 3 & 3 \\
\midrule
\rowcolor{col!20}\multicolumn{6}{c}{\textbf{Planning}}\\
\midrule
Samples & - & - & - & 1 & 1 \\
Temperature & - & - & - & 0.2 & 0.2 \\
Retries & - & - & - & 3 & 3 \\
\midrule
\rowcolor{col!20}\multicolumn{6}{c}{\textbf{Test case generation \& verification}}\\
\midrule
Benchmark-provided   & X & O & X & O & O \\
Generation  & X & X & O & O & O \\
Temperature & - & - & 0.2 & 0.2 & 0.2 \\
Retries & - & - & 3 & 3 & 3 \\
\bottomrule
\end{tabular}
\caption{\textbf{Baseline details.} Hyperparameter details for baselines and PaT.}
\end{table*}
\section{Baseline Details}
\label{sec:app_baseline}
\paragraph{Standard} \citep{gpt3}
represents the most direct approach to code generation, establishing the base capability of a model.
It performs a single, one-time generation attempt to produce the entire program from the problem specification.
For our code generation tasks, we follow a few-shot prompting protocol, providing the model with a small number of in-context examples to guide its output.
We generate a single candidate solution with a low temperature of $0.3$ to produce the most probable and deterministic result.
This baseline serves to measure the model's raw performance without any complex strategies like sampling, refinement, or decomposition.

\paragraph{Best-of-N} \citep{chen2021evaluating}
aims to improve performance by exploring a diverse set of potential solutions through sampling.
For each problem, it generates $N$ candidate programs using a high temperature to encourage variety.
To ensure a fair comparison with the trial phase of our PaT policy, we use the same sampling parameters: we set $N=5$ and use a temperature of $0.8$.
Each of the 5 candidates is then executed against the provided test suite, and the one that passes the most tests is selected as the final solution.
This method increases the probability of finding a correct solution at the cost of generating and evaluating multiple candidates.
Since MBPP does not have a test suite, we deterministically submit the first of the 5 generated candidates.

\paragraph{CodeT} \citep{chen2022codet}
is utilized from the implementation provided by the authors of FunCoder to ensure a faithful reproduction of their setup.
The process begins by sampling a pool of candidate solutions.
Following the hyperparameter settings of FunCoder, we used a larger sample size of $N=11$ with a temperature of $0.8$, which diverges from PaT and Best-of-N's $N=5$.
This is because CodeT's consensus-based ranking mechanism is highly dependent on a large and diverse candidate pool to function effectively; using a smaller sample size would artificially weaken this baseline.

\paragraph{FunCoder} \citep{chen2024divide}
serves as our primary baseline for the Planning-before-Trial (PbT) policy, and we use the official implementation to ensure a fair and accurate comparison.
The method operates as a rigid two-stage pipeline. 
First, a planner model decomposes the problem into a complete plan of helper functions using a temperature of $0.8$.
Only after this static plan is finalized does it proceed to the second stage, where it solves each subproblem using a consensus-based mechanism similar to CodeT, with $N=10$ and a temperature of $0.8$.
This rigid, plan-first approach, where the plan is fixed regardless of generation outcomes, provides a direct contrast to our adaptive PaT policy.

% \section{Analysis of Generated Test Cases}

% To assess the reliability of our self-verification mechanism, we quantitatively analyze generated test cases on HumanEval.
% PaT generates an average of $6.7$ test cases per problem, providing significantly broader coverage than standard public examples.
% Crucially, we examine the distribution of false positives, defined as incorrect test cases that fail valid solutions, to determine if verification noise is systemic.
% As illustrated in Figure~\ref{fig:test_noise_distribution}, the results confirm that the verification signal is robust for the vast majority of tasks.
% For instance, with Qwen3$_{\text{4B}}$, 63.4\% of test cases are completely error-free.
% Furthermore, instances of severe noise (3+ false positives) are concentrated in a small minority of problems.
% In these specific noisy instances, the pass rate inevitably saturates because the model cannot satisfy incorrect tests regardless of code correctness.
% This empirically justifies our plateau heuristic as a necessary early-stopping mechanism designed to prevent the generator from wasting computation by chasing these unattainable signals.

\section{Prompt Details}
\label{sec:app_prompt}
Our prompting strategy largely follows the one provided in the FunCoder \citep{chen2024divide} implementation to ensure a fair comparison. 
We use the same few-shot examples and sampling methods for both the generation (conquer) and planning (divide) phases.

% The single critical modification is in the decomposition prompt.
% The FunCoder prompt uses a conditional instruction, asking the planner to decompose a problem if it seems complex.
% In contrast, since our PaT policy only invokes the planner after a direct attempt has already failed, the problem's difficulty has been established.
% Therefore, our planner prompt is an unconditional command to decompose the problem.

\subsection{Prompt for Generation}
\begin{lstlisting}[style=pythonstyle3]
You are a programming copilot, you can solve a problem by writing Python functions. Your task is to:

  - For every turn, you need to write a Python function that returns the answer, based on current code (not code in chat history) and problem description.
  - Do not modify function name, arg names, docstring in given functions.
  - Consider reusing existing functions that are already implemented.
  - You can import libraries to better solve the problem.

<User>:

Current Code:

```python
\end{lstlisting}

\begin{lstlisting}[style=pythonstyle4]
def prime_factor(x: int) -> list:
    """get a list of prime factors of number $x$"""
    ret = []
    i = 1
    while i * i <= x:
        i += 1
        if x % i == 0 and is_prime(i):
            ret.append(i)
    return ret

def is_prime(x: int) -> bool:
    """determine $x$ is a prime number or not"""
    if x < 2:
        return False
    for i in range(2, int(x**0.5) + 1):
        if x % i == 0:
            return False
    return True

def get_common(a: list, b: list) -> list:
    """get common element in two list $a$ and $b$"""
    ret = []
    for item in a:
        if item in b:
            ret.append(item)
    return ret

def sum_common_factors(a: int, b: int) -> int:
    """Return the sum of all common prime factors of $a$ and $b$"""

    raise NotImplementedError()
\end{lstlisting}

\begin{lstlisting}[style=pythonstyle3]
'''

Let's think step by step and implement the following method `sum_common_factors` using existing functions to solve:
"Return the sum of all common prime factors of $a$ and $b$"

<Assistant>:

First, I need to get the prime factors of $a$ and $b$.
Second, I can use `for` loop to find common element in two factors list.
Finally, sum the common factor list and return the answer.
Here is the `sum_common_factors` function:

```python
\end{lstlisting}

\begin{lstlisting}[style=pythonstyle4]
def sum_common_factors(a: int, b: int) -> int:
    """Compute the sum of all common prime factors of $a$ and $b$"""
    factors_a = prime_factor(a)
    factors_b = prime_factor(b)
    common_factors = get_common(factors_a, factors_b)
    return sum(common_factors)
\end{lstlisting}

\begin{lstlisting}[style=pythonstyle3]
'''

<User>:

Current Code:

```python
{prev_code}
'''

Let's think step by step and implement the following method `{cur_func_name}` using existing functions to solve:
"{cur_func_doc}"
\end{lstlisting}

\subsection{Prompt for Planning}
\begin{lstlisting}[style=pythonstyle3]
You are a programming copilot, you can solve a problem by writing Python functions. Your task is to:
  - The previous attempt to direct implement the target function is failed, indicating its overall logic might be too complex to implement directly.
  - For every turn, you need to write a Python function that returns the answer based on Current Code (not code in chat history).
  - Do not modify function name, arg names, docstring in given functions.
  - You can import libraries to better solve the problem.
  - You can leave new function unimplemented for now, but write the function at the end of the code and comment what the function does.
  - Therefore, you must decompose it into multiple smaller, manageable helper functions.

<User>:

Current Code:
```python
\end{lstlisting}

\begin{lstlisting}[style=pythonstyle4]
def sum_common_factors(a: int, b: int) -> int:
    """Compute the sum of all common prime factors of $a$ and $b$"""
    raise NotImplementedError()
\end{lstlisting}

\begin{lstlisting}[style=pythonstyle3]
'''

Let's think step by step and complete the following Python function `sum_common_factors` that solves:
"Compute the sum of all common prime factors of $a$ and $b$"

<Assistant>:

First, I need to get the prime factors of $a$ and $b$.
Second, I can use `for` loop to find common element in two factors list.
Finally, sum the common factor list and return the answer.
Here is the `sum_common_factors` function:

```python
\end{lstlisting}

\begin{lstlisting}[style=pythonstyle4]
def sum_common_factors(a: int, b: int) -> int:
    """Compute the sum of all common prime factors of $a$ and $b$"""
    factors_a = prime_factor(a)
    factors_b = prime_factor(b)
    common_factors = get_common(factors_a, factors_b)
    return sum(common_factors)

def prime_factor(x: int) -> list:
    """get a list of prime factors of number $x$"""
    raise NotImplementedError()

def get_common(a: list, b: list) -> list:
    """get common element in two list $a$ and $b$"""
    raise NotImplementedError()
\end{lstlisting}

\begin{lstlisting}[style=pythonstyle3]
'''

<User>:

Current Code:
```python
\end{lstlisting}
\begin{lstlisting}[style=pythonstyle4]
def sum_common_factors(a: int, b: int) -> int:
    """Compute the sum of all common prime factors of $a$ and $b$"""
    factors_a = prime_factor(a)
    factors_b = prime_factor(b)
    common_factors = get_common(factors_a, factors_b)
    return sum(common_factors)

def get_common(a: list, b: list) -> list:
    """get common element in two list $a$ and $b$"""
    raise NotImplementedError()
\end{lstlisting}

\begin{lstlisting}[style=pythonstyle3]
'''

Let's think step by step and complete the following Python function `get_common` that solves:
"get common element in two list $a$ and $b$"

<Assistant>:

Here is the `get_common` function:

```python
\end{lstlisting}

\begin{lstlisting}[style=pythonstyle4]
def get_common(a: list, b: list) -> list:
    """get common element in two list $a$ and $b$"""
    ret = []
    for item in a:
        if item in b:
            ret.append(item)
    return ret
\end{lstlisting}

\begin{lstlisting}[style=pythonstyle3]
'''

<User>:

Current Code:
```python
{prev_code}
'''

Let's think step by step and complete the following Python function `{cur_func_name}` that solves:
"{cur_func_doc}"
\end{lstlisting}

\end{document}